\documentclass[11pt,a4paper]{article}

\usepackage[T1]{fontenc}
\usepackage[utf8]{inputenc}
\usepackage{mathpazo}
\usepackage{microtype}
\usepackage[margin=2.5cm,headheight=14pt]{geometry}
\usepackage{fancyhdr}
\usepackage{float,placeins,threeparttable}
\usepackage{amsmath,amssymb,amsthm,mathtools}
\usepackage{tikz,pgfplots}
\pgfplotsset{compat=1.18}
\usetikzlibrary{arrows.meta,positioning,calc}
\usepackage{booktabs,array}
\usepackage[dvipsnames,table]{xcolor}
\definecolor{lightblue}{HTML}{DBEAFE}
\definecolor{lightgray}{HTML}{F3F4F6}
\definecolor{darkblue}{HTML}{1E3A5F}
\definecolor{thmcolor}{HTML}{1D4ED8}
\usepackage{enumitem}
\usepackage{caption,subcaption}
\usepackage[numbers,sort&compress]{natbib}
\usepackage[colorlinks=true,
            citecolor=thmcolor,
            linkcolor=darkblue,
            urlcolor=darkblue]{hyperref}
\usepackage{cleveref}

\theoremstyle{plain}
\newtheorem{theorem}{Theorem}[section]
\newtheorem{proposition}{Proposition}[section]

\newtheorem{apptheorem*}{Theorem}
\newtheorem{appprop*}{Proposition}

\crefname{theorem}{Theorem}{Theorems}
\crefname{proposition}{Proposition}{Propositions}
\crefname{section}{Section}{Sections}
\crefname{figure}{Figure}{Figures}
\crefname{table}{Table}{Tables}

\DeclareMathOperator{\Tr}{Tr}
\newcommand{\POVM}{\mathcal{H}_{\mathrm{POVM}}}
\newcommand{\calF}{\mathcal{F}}

\newcommand{\nstar}{n^{*}}
\newcommand{\Sstar}{S^{*}}

\newcommand{\gapbound}{\mathcal{G}}

\title{%
  Measurement-Budget Allocation in Quantum Learning with Finite-Shot Generalization Guarantees}
\author{%
  Ferhat Ozgur Catak\\[2pt]
  \small Department of Electrical Engineering and Computer Science\\
  \small University of Stavanger, Norway\\
  \small \texttt{f.ozgur.catak@uis.no}}
\date{\small\today}

\begin{document}
\maketitle

% ── Abstract ─────────────────────────────────────────────────
\begin{abstract}
On near-term quantum hardware, estimating a Born probability requires repeated circuit executions. A quantum learning experiment with a fixed measurement budget $B$ must decide how many distinct training states $n$ to use and how many shots $S$ to allocate to each state. We study this trade-off for binary quantum classifiers whose measurement operator $M$ is learned from the training sample, under the requirement that $M$ is selected independently of the measurement outcomes used to estimate its empirical risk, where the ideal score is $\Tr(M\rho)$. We prove a distribution-free generalization bound for bounded, $L$-Lipschitz losses that separates the finite-sample and finite-shot contributions. The sample term scales as $\sqrt{d/n}$, while the shot term scales as $\sqrt{(\log n)/S}$; under the constraint $B=nS$, these two terms move in opposite directions. Minimizing a conservative closed-form surrogate of the bound gives the allocation rule $\nstar = 2\sqrt{2dB/\log(4B/\delta)}$ and $\Sstar = B/\nstar$. This surrogate has the same asymptotic scaling as the exact minimizer and, on all budgets we evaluate, its bound value is within $0.1\%$ of the exact numerical minimum; it yields a worst-case rate of $B^{-1/4}$. The guarantee covers the worst case over the full class of binary quantum measurements. We complement the theory with PennyLane simulations using 2-qubit and 4-qubit variational quantum circuits on nine synthetic binary classification benchmarks. In all tested configurations, the one-sided empirical generalization gap remains below the theoretical bound. However, at the simulated budgets the bound exceeds one, the maximum possible gap, and it becomes non-vacuous only for larger budgets (e.g., $B\gtrsim 5000$ for $d=4$). We further show that, for a loss that is concave in the score, the empirical gap is U-shaped under a fixed budget and the measured shot contribution matches its analytical value, whereas for the absolute loss shot noise is unbiased and the gap does not increase in the shot-limited regime. The rule can be used to plan finite-shot evaluations before an experiment is run.

\end{abstract}


% ============================================================
\section{Introduction}
\label{sec:intro}
% ============================================================

On current quantum hardware, measurement results are obtained from repeated circuit executions, and the number of executions is limited by time and cost~\cite{preskill2018quantum,RevModPhys.94.015004,cerezo2021variational}. Gate times, coherence limits, and access restrictions on cloud-based quantum processors all limit the number of circuit executions available for a given experiment. For quantum learning, an important constraint is the \emph{shot cost}, i.e., the number of repeated circuit executions required to estimate the expectation value of a measurement.

A quantum classifier assigns a score to a quantum state $\rho$ through a measurement operator $M$. The ideal score is the Born probability 

$$
p_M(\rho) = \Tr(M\rho) \in [0,1].
$$
This quantity is not observed directly. Instead, each circuit execution produces a binary outcome $Z \in \{0,1\}$ with $\Pr(Z=1)=\Tr(M\rho)$, and the probability must be estimated from repeated measurements. Estimating $\Tr(M\rho)$ to accuracy $\varepsilon$ requires on the order of $1/(4\varepsilon^2)$ independent shots. Recent work on quantum neural networks and quantum learning has also emphasized that finite sampling noise is a statistical effect, even on otherwise ideal quantum devices, because expectation-value outputs must be estimated from repeated measurements~\cite{Kreplin2024reductionoffinite,2024arXiv241014654H}.

We treat shots as a statistical resource, separate from the number of training states. In a typical quantum learning experiment, the total measurement budget is
$$
B = n \times S,
$$
where $n$ denotes the number of distinct training states and $S$ the number of shots assigned to each state. In our setting, $B$ is the number of binary measurement outcomes used to estimate the empirical risk of the learned classifier on its training states. This cost model is appropriate when circuit executions dominate the cost of an experiment, as on metered cloud platforms and in shot-based simulation. When preparing a new input state is much more expensive than repeating a measurement, a weighted cost model is required; we discuss this case in \cref{sec:discussion}.

In many quantum machine learning studies, the common choice is to use a relatively large number of shots and a small number of training states~\cite{biamonte2017qml,schuld2018supervised,Kubler2020adaptiveoptimizer,2020arXiv200406252A}. This is partly because circuit simulation is expensive and repeated measurements reduce the variance of each estimate. However, this choice is not usually derived from a statistical allocation principle. Under a fixed total budget $B$, it can place the experiment in a sample-limited regime, where increasing the number of training states would reduce the generalization error more effectively than the corresponding decrease in shots would increase the estimation error.

The central question is how to allocate a fixed budget $B=nS$ between $n$ and $S$. A useful allocation must balance two sources of error: the error due to observing only finitely many training states, and the error due to estimating each quantum score from finitely many measurement outcomes.

\medskip
\noindent\textbf{Many training states, few shots per state.}
If $n$ is increased while the total budget $B$ is fixed, then $S=B/n$ decreases. Each estimate $\hat{p}_{M,S}(\rho_i)$ is then based on fewer measurement outcomes and becomes less reliable. In the limiting case $S=1$, the estimate is a single Bernoulli outcome with variance at most $1/4$, so the observed training scores are noisy.

\medskip
\noindent\textbf{Few training states, many shots per state.}
If $S$ is increased under the same fixed budget, then $n=B/S$ decreases. The Born probabilities are estimated more accurately, but the learner observes fewer distinct training states. In the limiting case $n=1$, shot noise can be made small, but the finite-sample component of the generalization gap remains of order $O(1)$.

\medskip
The best allocation lies between these two regimes. Its position depends on the Hilbert-space dimension $d=2^k$ for $k$ qubits, the total budget $B$, and the confidence level $\delta$.

Standard generalization theory does not resolve the allocation problem, for two reasons.

The first difficulty is that the observed score $\hat{p}_{M,S}(\rho_i)$ is only an estimate of the Born probability $p_M(\rho_i)=\Tr(M\rho_i)$. Classical bounds usually assume that hypothesis values are evaluated exactly. In the finite-shot setting, the empirical loss is computed from noisy measurement estimates, and the size of this perturbation depends explicitly on $S$.

The second difficulty is that classical learning theory treats the number of training examples $n$ as the main statistical resource. It does not provide a rule for choosing both $n$ and $S$ under the constraint $nS=B$. Existing quantum generalization bounds~\cite{caro2022generalization,haug2024generalization} either work in the infinite-shot regime or keep the number of shots fixed. They do not address how a finite measurement budget should be split between additional training states and more accurate measurement estimates.

We adapt the Rademacher-complexity framework~\cite{bartlett2002rademacher} to account for finite-shot quantum measurements. The first part of the analysis concerns the ideal, infinite-shot setting. For the measurement class
$$
\POVM=\{M:0\leq M\leq I\},
$$
we show that
$$
\hat{\mathcal{R}}_n^{(\infty)}(\POVM)\leq \sqrt{d/n}.
$$
The proof uses the linearity of $\Tr(M\cdot)$ and the fact that, for a fixed Hermitian matrix, the optimizing measurement is the projector onto its positive eigenspace. Standard nuclear- and Frobenius-norm bounds then give the stated dependence on the Hilbert-space dimension $d$ and the number of training states $n$.

The second part of the analysis accounts for finite measurement shots. We bound the deviation between the ideal empirical risk and the shot-estimated empirical risk using Hoeffding concentration together with the Lipschitz property of the loss. Combining the sample and shot deviations yields a bound with two separate contributions: a finite-sample term controlled by $n$, and a finite-shot term controlled by $S$.

The resulting bound (Theorem~\ref{thm:main}) is

\begin{equation}
  \underbrace{R_\ell(M)}_{\text{population risk}}
  \;\leq\;
  \underbrace{\hat{R}_{n,S,\ell}(M)}_{\text{observed empirical risk}}
  \;+\;
  \underbrace{2L\sqrt{\frac{d}{n}}}_{\text{sample term}}
  \;+\;
  \underbrace{L\sqrt{\frac{\log(4n/\delta)}{2S}}}_{\text{shot term}}
  \;+\;
  \underbrace{\sqrt{\frac{\log(2/\delta)}{2n}}}_{\text{confidence term}}
  \label{eq:intro-bound}
\end{equation}

Under $B=nS$, the sample term decreases and the shot term increases with $n$, so the bound is U-shaped in $n$. Minimizing a closed-form surrogate gives the allocation rule in C2 below; $\nstar$ grows as $\sqrt{B/\log B}$ and $\Sstar$ as $\sqrt{B\log B}$, and the bound value at this allocation decays as $(d\log B/B)^{1/4}$ (see \cref{sec:discussion} for the interpretation of this rate).

The measurement class $\POVM=\{M:0\leq M\leq I\}$ plays a role similar to bounded linear classifiers of the form $\{x\mapsto \langle w,x\rangle:\|w\|\leq \Lambda\}$. In both cases, the hypothesis is specified by a constrained linear functional, and the corresponding Rademacher complexity scales as $O(1/\sqrt{n})$, up to dimension-dependent factors. Unlike the classical case, the value $\Tr(M\rho)$ is not observed directly; it is estimated from binary measurement outcomes, whereas $\langle w,x_i\rangle$ is available exactly.

Shot noise can be viewed as a measurement-induced analogue of noise in the empirical loss. Its size decreases with the number of shots, typically at order $O(1/\sqrt{S})$. Unlike ordinary label noise, however, shot noise is not fixed by the data-generating process. It is partly controlled by the experimental design, since the practitioner chooses how many shots $S$ to allocate to each state.

The guarantee has the same distribution-free character as classical Rademacher bounds. It applies to any distribution $\mathcal{P}$ over quantum states and labels, any loss with values in $[0,1]$ that is $L$-Lipschitz in the score, and any measurement $M\in\POVM$ chosen independently of the shot outcomes used to compute the empirical risk. $M$ may be learned from the training states and labels; the exact conditions are stated after Theorem~\ref{thm:main}.

Our contributions are as follows.
\begin{enumerate}[label=\textbf{C\arabic*.},leftmargin=*]

\item \textbf{Finite-shot generalization bound.}
Theorem~\ref{thm:main} gives a distribution-free bound~\eqref{eq:intro-bound} for quantum measurement hypotheses. The bound separates the finite-sample contribution, of order $O(\sqrt{d/n})$, from the finite-shot contribution, of order $O(\sqrt{(\log n)/S})$.

\item \textbf{Conservative allocation rule under a fixed measurement budget.}
Under the constraint $B=nS$, Theorem~\ref{thm:allocation} derives a closed-form allocation rule for a surrogate of the finite-shot evaluation bound:
$$
n^{\ast}=2\sqrt{\frac{2dB}{\log(4B/\delta)}},
\qquad
S^{\ast}=\frac{B}{n^{\ast}}.
$$
This rule minimizes the surrogate upper bound for fixed or independently selected measurements. The rule is independent of the Lipschitz constant $L$, makes the dependence of the split on $d$, $B$, and $\delta$ explicit, and achieves a bound value within $0.1\%$ of the exact numerical minimum on all budgets we evaluate.

\item \textbf{Sample-limited and shot-limited regimes.}
The analysis identifies three regimes: sample-limited, shot-limited, and balanced. These regimes clarify how the bound changes when budget is moved between additional training states and additional shots per state. They also explain why configurations with large $S$ and small $n$, which are common in QML studies, may be sample-limited under a fixed budget.

\item \textbf{Empirical consistency checks.}
We test the bound using PennyLane simulations with 2-qubit ($d=4$) and 4-qubit ($d=16$) variational quantum circuits on nine synthetic binary classification benchmarks. Across the tested values of $n$, $S$, and $B$, the one-sided empirical generalization gaps remain below the theoretical bound. An additional experiment with a concave score loss shows the U-shaped trade-off in the empirical gap and explains why it is absent for losses that are affine in the score.

\end{enumerate}

The PennyLane-based experimental code and synthetic-dataset generation scripts supporting the findings of this study are publicly available on GitHub\footnote{\url{https://github.com/ocatak/quantum-measurement-budget-allocation}}.

The remainder of the paper is structured as follows. \cref{sec:related} reviews related work. \cref{sec:system} introduces the system model, the two-gap decomposition, the main generalization bound in Theorem~\ref{thm:main}, its scope, and the surrogate allocation rule in Theorem~\ref{thm:allocation}. \cref{sec:experiments} presents the experimental results. \cref{sec:discussion} discusses implications, limitations, and open problems. \cref{sec:conclusion} concludes the paper.

% ============================================================
\section{Related Work}
\label{sec:related}
% ============================================================

We organize prior work around three capabilities that are required to answer
the budget allocation question:
(Q1)~Does the framework account for finite measurement shots?
(Q2)~Does it jointly optimize $n$ (training states) and $S$ (shots per state)?
(Q3)~Does it yield a closed-form allocation rule for fixed budget $B = nS$?
Table~\ref{tab:related} summarizes the comparison; the paragraphs below
describe the focus of each prior approach and explain why the allocation
question under finite $B$ lies outside its scope.

\begin{table}[h]
\centering
\caption{Comparison of quantum generalization frameworks along three axes
  relevant to the budget allocation problem.
  \checkmark~= addressed; $\times$~= not in scope; $\sim$~= partially.
  Q1: models finite measurement shots;
  Q2: analyzes the joint dependence on $n$ and $S$;
  Q3: provides a closed-form allocation rule under a fixed budget $B=nS$.}
\label{tab:related}
{%
\setlength{\tabcolsep}{6pt}\renewcommand{\arraystretch}{1.15}
\begin{tabular}{@{}llccc@{}}
\toprule
\textbf{Work} & \textbf{Method}
  & \textbf{Q1} & \textbf{Q2} & \textbf{Q3} \\
\midrule
Caro et al.~\cite{caro2022generalization}
  & Covering numbers    & $\times$ & $\times$ & $\times$ \\
Haug \& Kim~\cite{haug2024generalization}
  & Quantum Fisher information & $\times$ & $\times$ & $\times$ \\
Huang et al.~\cite{huang2021information}
  & Information-theoretic & $\times$ & $\times$ & $\times$ \\
Gil-Fuster et al.~\cite{gilfuster2024rethinking}
  & Randomization tests  & $\sim$   & $\times$ & $\times$ \\
Bshouty \& Jackson~\cite{bshouty1998learning}
  & Quantum PAC         & $\times$ & $\times$ & $\times$ \\
Berberich et al.~\cite{berberich2024robust}
  & Lipschitz bounds    & $\times$ & $\times$ & $\times$ \\
Recio-Armengol et al.~\cite{recioarmengol2025singleshot}
  & Single-shot classifiers & \checkmark & $\times$ & $\times$ \\
Hu \& Gao~\cite{2024arXiv241014654H}
  & Statistical mechanics & \checkmark & \checkmark & $\times$ \\
\rowcolor{lightblue!40}
\textbf{This work}
  & Rademacher complexity & \checkmark & \checkmark & \checkmark \\
\bottomrule
\end{tabular}}
\end{table}

\paragraph{Covering numbers for parameterized quantum circuits~\cite{caro2022generalization}.}
Caro et al.\ derive generalization bounds of order $O(\sqrt{T/n})$ for parameterized quantum circuits with $T$ trainable gates, via covering numbers in the infinite-shot limit ($S \to \infty$).  It studies finite-sample generalization for PQC-induced function classes, whereas we study finite-shot corrections for the full measurement class $\POVM$. The covering-number metric is defined on exact hypothesis values and does not extend to shot-estimated approximations, so the framework does not yield guidance on $S$ or on the trade-off under fixed $B$.

\paragraph{Quantum Fisher information~\cite{haug2024generalization}.}
Haug and Kim use the rank of the data quantum Fisher information metric to study circuit capacity and the number of training states needed for generalization. QFI is a local quantity that characterizes parameter-space sensitivity and is well suited to the study of trainability and barren plateaus.  Because QFI is defined in the infinite-shot limit and provides a bound at a single parameter point rather than over the full class $\POVM$, it does not yield allocation guidance for finite $S$.  QFI bounds inform circuit design; our bound informs measurement allocation.

\paragraph{Information-theoretic limits~\cite{huang2021information}.}
Huang, Kueng, and Preskill establish bounds on quantum advantage in terms of the number of training examples required.  Their budget is a count of training examples; ours is a count of measurement outcomes.

\paragraph{Quantum PAC learning~\cite{bshouty1998learning,arunachalam2017survey}.}
The quantum PAC learning literature studies classical concept classes learned from quantum example oracles, where the quantum resource is the oracle access model.  We instead study quantum hypothesis classes (measurement operators $M \in \POVM$) on quantum state data, with finite-shot accounting.

\paragraph{Rethinking quantum generalization~\cite{gilfuster2024rethinking}.}
Gil-Fuster et al.\ show that classical capacity measures (VC dimension, covering numbers) do not explain the generalization behavior of quantum models. For the full measurement class $\POVM$, we show that the Rademacher complexity is at most $\sqrt{d/n}$ in the infinite-shot limit and add a shot-noise term of order $\sqrt{(\log n)/S}$. Rademacher complexity has also been used to quantify the statistical complexity of quantum circuits and parameterized quantum models, including encoding-dependent generalization bounds for PQCs and complexity analyses of noisy quantum circuits~\cite{PhysRevA.105.062431,2021arXiv210303139B,Caro2021encodingdependent}. Our analysis differs by considering the full binary measurement class  $\mathcal{H}_{\mathrm{POVM}}$ and by adding an explicit finite-shot perturbation term.

{\paragraph{Lipschitz-based robustness and generalization~\cite{berberich2024robust}.}
Berberich et al.\ derive Lipschitz bounds for quantum models with trainable data encodings and show that the norm of the encoding parameters controls both robustness to input perturbations and generalization; they also propose Lipschitz regularization as a training objective. Their Lipschitz constant is defined with respect to the classical input. In contrast, we use the Lipschitz constant of the loss with respect to the score, and for a different purpose, namely to propagate the finite-shot estimation error of the scores to the empirical risk. Their parameter-dependent complexity measures could replace the class-level term $\sqrt{d/n}$ in our bound (see OP2).}

{\paragraph{Single-shot quantum machine learning~\cite{recioarmengol2025singleshot}.}
Recio-Armengol et al.\ study the randomness that measurements introduce into the \emph{predictions} of a quantum classifier and give conditions under which a single run of the model returns the correct label with high probability. They relate this behavior to the distinguishability of the embedded states and to the depth of the embedding circuit. Their work addresses the number of shots required at inference time for a given model. We address the number of shots used to \emph{estimate the training risk}, and how a fixed measurement budget should be split between this estimation and the number of training states. In our notation, a single-shot classifier has scores $\Tr(M\rho)$ close to $0$ or $1$; for such models the per-state shot error is small and the shot term of Theorem~\ref{thm:main} is conservative.}

{\paragraph{Sampling noise and generalization~\cite{2024arXiv241014654H}.}
Hu and Gao consider the setting closest to ours. Using tools from statistical mechanics and the Eigentask Learning framework, they compute the training and generalization errors of a quantum learning model when both the number of training inputs and the number of measurement shots are finite, and they propose a feature-space tuning method that reduces the effect of sampling noise. Their analysis characterizes the typical (average-case) error of linear read-outs of quantum features. Our result is a high-probability, distribution-free upper bound over the full measurement class $\POVM$. Because its two terms have a closed form, the budget-constrained problem under $B=nS$ can be solved analytically. Their analysis thus predicts the typical error for a given $(n,S)$, whereas ours provides a worst-case guarantee and a rule for selecting $(n,S)$.}

\paragraph{Shadow tomography~\cite{aaronson2018shadow,huang2020predicting}.}
The signed-average matrix $\bar{\rho}_{\boldsymbol{\sigma}} = \frac{1}{n}\sum_i \sigma_i\rho_i$ arising in the Rademacher complexity proof is related to the classical shadow estimators used in shadow tomography.  Formalizing this connection and determining whether shadow techniques could tighten our bound is left as an open problem (OP8). Classical-shadow methods show that, for suitable randomized measurement schemes, many linear properties of a quantum state can be predicted from few measurements~\cite{huang2020predicting}. 

{\paragraph{Summary.}
Prior work on quantum generalization mostly studies the infinite-shot or fixed-shot regime and does not address the joint choice of $n$ and $S$ under a fixed budget $B=nS$. Existing bounds either assume exact access to the scores~\cite{caro2022generalization,haug2024generalization,berberich2024robust} or give the average-case effect of sampling noise for a specific model family~\cite{2024arXiv241014654H}. Our bound keeps the finite-sample and finite-shot errors as separate terms, which makes the budget-constrained choice of $(n,S)$ solvable in closed form. The bound relies on standard techniques; what is new is that the measurement budget is treated as a second design variable alongside the number of training states.}

% ============================================================
\section{System Model and Main Results}
\label{sec:system}
% ============================================================

\subsection{Quantum Learning Setup}

\paragraph{Quantum training examples.}
A quantum training example is a pair $(\rho_i, y_i)$ where $\rho_i$ is a \emph{density matrix}, a $d \times d$ positive semidefinite operator with unit trace and $y_i \in \{0,1\}$ is its label.  The training set $\mathcal{D} = \{(\rho_i, y_i)\}_{i=1}^n$ is drawn i.i.d.\ from an unknown distribution $\mathcal{P}$ over $\mathcal{S}(\mathcal{V}) \times \{0,1\}$, where $\mathcal{S}(\mathcal{V})$ denotes the set of density matrices on $\mathcal{V}=\mathbb{C}^d$.

\paragraph{Quantum hypotheses.}
A quantum hypothesis is a measurement operator $M$ with $0 \leq M \leq I$, where $I$ is the $d\times d$ identity.  Its prediction score on state $\rho$ is the Born probability $p_M(\rho) = \Tr(M\rho) \in [0,1]$.  The full binary measurement hypothesis class is:
\[
  \POVM = \{M : 0 \leq M \leq I,\; M \in \mathbb{C}^{d\times d}\}
\]
For $k$ qubits, $d = 2^k$.

{
\paragraph{Learning algorithm and hypothesis.}
In a variational quantum classifier, the hypothesis is parameterized as
\[
  M(\boldsymbol{\theta}) = U(\boldsymbol{\theta})^{\dagger}\,\Pi\,U(\boldsymbol{\theta}),
\]
where $U(\boldsymbol{\theta})$ is a trainable circuit and $\Pi$ is a fixed projector, e.g., $\Pi=\lvert 1\rangle\langle 1\rvert\otimes I$ for a read-out of the first qubit. Since $0\le M(\boldsymbol{\theta})\le I$ for every $\boldsymbol{\theta}$, the model class $\mathcal{H}_{\mathrm{VQC}}=\{M(\boldsymbol{\theta})\}$ is a subclass of $\POVM$. A learning algorithm $\mathcal{A}$ maps the training sample $\mathcal{D}=\{(\rho_i,y_i)\}_{i=1}^n$ to a parameter vector $\boldsymbol{\theta}$, and hence to a hypothesis $M=\mathcal{A}(\mathcal{D})$. Our results make no assumption on $\mathcal{A}$: they hold for any $M\in\POVM$ returned by $\mathcal{A}$, provided that the measurement outcomes used to estimate its empirical risk are not also used by $\mathcal{A}$ (see the scope paragraph after Theorem~\ref{thm:main}). The experiments in Section~\ref{sec:experiments} instantiate $\mathcal{A}$ as gradient descent with exact (noiseless) gradients and estimate the empirical risk of the returned hypothesis from $S$ fresh shots per training state.}

\paragraph{Finite-shot estimation.}
Performing $S$ independent applications of $\{M, I-M\}$ on state $\rho_i$ produces outcomes $Z_{i,1}, \ldots, Z_{i,S} \in \{0,1\}$.  The shot-estimated score is:
\begin{equation}
  \hat{p}_{M,S}(\rho_i) = \frac{1}{S}\sum_{s=1}^S Z_{i,s},
  \qquad
  \mathbb{E}[\hat{p}_{M,S}(\rho_i)] = \Tr(M\rho_i),
  \qquad
  \mathrm{Var}[\hat{p}_{M,S}(\rho_i)] \leq \frac{1}{4S}
  \label{eq:shot-est}
\end{equation}

\paragraph{Loss function.}
Formally, a loss is a function $\ell:[0,1]\times\{0,1\}\to[0,1]$ that maps a score $p\in[0,1]$ and a label $y$ to a value in $[0,1]$. We assume that $\ell$ is \emph{bounded}, i.e., takes values in $[0,1]$, and \emph{$L$-Lipschitz in the score}: $|\ell(p_1, y) - \ell(p_2, y)| \leq L|p_1 - p_2|$ for all $p_1,p_2\in[0,1]$ and $y\in\{0,1\}$. Both assumptions are used in the analysis: boundedness is required for the confidence term, and $L$ scales both the sample term and the shot term. For a loss with range $[0,c]$, the confidence term scales by $c$.  The squared loss $(p-y)^2$ satisfies these assumptions with $L=2$. The absolute score loss $\ell(p,y)=|p-y|$, which coincides on $[0,1]$ with the rescaled binary hinge loss used in Section~\ref{sec:experiments}, is bounded in $[0,1]$ and satisfies the condition with $L=1$. The hard-threshold $0/1$ loss is excluded because its discontinuity near the decision boundary makes the finite-shot analysis intractable without an additional margin assumption.

\subsection{Three Risk Quantities}

The learning pipeline involves three risk quantities, listed in Table~\ref{tab:three-risks}. Of these, only the shot-estimated empirical risk can be measured directly on a quantum device. The aim is to control the population risk using this observable quantity.

\begin{table}[h]
\centering
\caption{The three risk quantities used in the analysis.}
\label{tab:three-risks}
\setlength{\tabcolsep}{7pt}\renewcommand{\arraystretch}{1.4}
\begin{tabular}{@{}llll@{}}
\toprule
\textbf{Quantity} & \textbf{Symbol} & \textbf{Definition} & \textbf{Observable?} \\
\midrule
Population risk
  & $R_\ell(M)$
  & $\mathbb{E}_{(\rho,y)\sim\mathcal{P}}[\ell(p_M(\rho),y)]$
  & No (requires the full distribution) \\
Ideal empirical risk
  & $\hat{R}_{n,\ell}(M)$
  & $\frac{1}{n}\sum_i \ell(p_M(\rho_i), y_i)$
  & No (requires infinite shots) \\
Shot-estimated empirical risk
  & $\hat{R}_{n,S,\ell}(M)$
  & $\frac{1}{n}\sum_i \ell(\hat{p}_{M,S}(\rho_i), y_i)$
  & \textbf{Yes} \\
\bottomrule
\end{tabular}
\end{table}

\subsection{Two-Gap Structure and Main Bound}

Figure~\ref{fig:pipeline} illustrates how the two statistical gaps arise in the quantum learning pipeline.

\begin{figure}[!h]
\centering
\resizebox{\textwidth}{!}{% Quantum learning pipeline — grayscale, professional typesetting
% Requires: tikz, \usetikzlibrary{arrows.meta,calc,positioning}
%
% Usage: \input{fig_pipeline.tex}  (wrap in figure/tikzpicture as needed)

\usetikzlibrary{arrows.meta,calc,positioning}

% ── Grayscale palette ──────────────────────────────────────────────
\definecolor{clbg}  {gray}{0.95}   % classical box fill
\definecolor{clbdr} {gray}{0.52}   % classical border
\definecolor{qtbg}  {gray}{0.82}   % quantum box fill
\definecolor{qtbdr} {gray}{0.32}   % quantum border
\definecolor{gapbg} {gray}{0.98}   % gap annotation fill
\definecolor{gapbdr}{gray}{0.58}   % gap annotation border
\definecolor{thmbg} {gray}{0.91}   % theorem box fill
\definecolor{thmbdr}{gray}{0.18}   % theorem box border
\definecolor{darrow}{gray}{0.52}   % dashed arrow colour
\definecolor{parrow}{gray}{0.12}   % pipeline arrow colour

\begin{tikzpicture}[
  font=\small\sffamily,
  %
  % ── Box styles ──────────────────────────────────────────────────
  % Classical pipeline box
  cbox/.style={
    draw=clbdr, fill=clbg, rounded corners=4pt,
    minimum width=2.2cm, minimum height=0.95cm,
    text centered, align=center,
    inner sep=5pt, line width=0.55pt
  },
  % Quantum pipeline box (darker fill + heavier border)
  qbox/.style={
    draw=qtbdr, fill=qtbg, rounded corners=4pt,
    minimum width=2.4cm, minimum height=0.95cm,
    text centered, align=center,
    inner sep=5pt, line width=0.9pt
  },
  % Gap annotation box (dashed)
  gbox/.style={
    draw=gapbdr, fill=gapbg, rounded corners=5pt, dashed,
    minimum width=3.5cm, minimum height=1.35cm,
    align=center, inner sep=7pt, line width=0.55pt
  },
  % ── Arrow styles ────────────────────────────────────────────────
  % Solid pipeline arrow
  arr/.style={
    -{Stealth[length=4.5pt, width=3.2pt]},
    color=parrow, line width=0.75pt
  },
  % Dashed explanatory arrow
  darr/.style={
    -{Stealth[length=3.5pt, width=2.5pt]},
    color=darrow, line width=0.55pt, dashed
  },
]

% ── TOP PIPELINE ROW ────────────────────────────────────────────────
\node[cbox] (A) at (0, 0)
  {Quantum example\\$(\rho_i, y_i)$};
\node[cbox, right=0.55cm of A] (B)
  {Hypothesis\\$M \in \mathcal{H}_{\mathrm{POVM}}$};
\node[cbox, right=0.55cm of B] (C)
  {Ideal score\\$\mathrm{Tr}(M\rho_i)$};
\node[qbox, right=0.55cm of C] (D)
  {Shot estimate\\$\hat{p}_{M,S}(\rho_i)$};

% ── BOTTOM PIPELINE ROW ─────────────────────────────────────────────
%   Anchored so D and E share the same horizontal centre.
\node[qbox, below=1.0cm of D] (E)
  {Shot-est.\ emp.\ risk\\$\widehat{R}_{n,S,\ell}(M)$};
\node[cbox, left=0.65cm of E,
      minimum width=2.35cm] (F)
  {Ideal emp.\ risk\\$\widehat{R}_{n,\ell}(M)$};
\node[cbox, left=0.65cm of F,
      minimum width=2.55cm] (G)
  {Population score risk\\$R_\ell(M)$};

% ── PIPELINE ARROWS ─────────────────────────────────────────────────
\draw[arr] (A) -- (B);
\draw[arr] (B) -- (C);
\draw[arr] (C) -- (D);
\draw[arr] (D) -- (E);   % vertical: quantum measurement
\draw[arr] (E) -- (F);
\draw[arr] (F) -- (G);

% ── GAP ANNOTATION BOXES ────────────────────────────────────────────
%   Finite-sample gap (classical)
\node[gbox, below=1.15cm of $(G.south)!0.4!(F.south)$,
      minimum width=3.8cm] (FG)
  {\textbf{Finite-sample gap}\\[3pt]
   $\displaystyle\leq 2L\sqrt{\tfrac{d}{n}}
     + \sqrt{\tfrac{\log(1/\delta)}{2n}}$\\[2pt]
   \textit{\footnotesize\color{clbdr}classical analogue}};

%   Finite-shot gap (quantum)
\node[gbox, below=1.15cm of $(F.south)!0.7!(E.south)$,
      minimum width=3.2cm] (SG)
  {\textbf{Finite-shot gap}\\[3pt]
   $\displaystyle\leq L\sqrt{\tfrac{\log(2n/\delta)}{2S}}$\\[2pt]
   \textit{\footnotesize\color{qtbdr}quantum-native}};

% ── EXPLANATORY DASHED ARROWS ────────────────────────────────────────
\draw[arr] (G) -- (FG);
\draw[arr] (F) -- (FG);
\draw[arr] (F) -- (SG);
\draw[arr] (E) -- (SG);

% ── THEOREM BOX ─────────────────────────────────────────────────────
\node[
  draw=thmbdr, fill=thmbg,
  rounded corners=4pt, line width=1.0pt,
  below=1.0cm of $(FG.south)!0.5!(SG.south)$,
  minimum width=12.0cm, minimum height=1.15cm,
  text width=11.7cm, align=center
] (THM) {%
  \textbf{Theorem~~\ref{thm:main}:}\\[4pt]
  $\displaystyle
    R_\ell(M)
    \;\leq\;
    \widehat{R}_{n,S,\ell}(M)
    \;+\; 2L\!\sqrt{\dfrac{d}{n}}
    \;+\; L\!\sqrt{\dfrac{\log(4n/\delta)}{2S}}
    \;+\; \sqrt{\dfrac{\log(2/\delta)}{2n}}$%
};

\end{tikzpicture}}
\caption{The quantum learning pipeline and its two sources of error. The finite-sample gap arises because the learner observes only $n$ training states from the data distribution. The finite-shot gap arises because each score $\Tr(M\rho_i)$ is estimated from $S$ binary measurement outcomes. Theorem~\ref{thm:main} bounds the combined effect of these two errors.}
\label{fig:pipeline}
\end{figure}

The capacity of $\POVM$ under random labeling is characterized by the following result, proved in Appendix~\ref{app:prop-idealrad}.

\begin{proposition}[Capacity of the quantum measurement class]\label{prop:idealrad}
For $\POVM$ on a $d$-dimensional Hilbert space and any $n$ training states:
\begin{equation}
  \hat{\mathcal{R}}_n^{(\infty)}(\POVM) \leq \sqrt{\frac{d}{n}}
  \label{eq:rad-bound}
\end{equation}
\end{proposition}

\noindent\textit{Proof sketch.}
For a fixed sign vector $\boldsymbol{\sigma}$, define
\begin{equation}
    \bar{\rho}_{\boldsymbol{\sigma}}=\frac{1}{n}\sum_i \sigma_i\rho_i
\end{equation}
By linearity of the trace, the supremum over $M$ is an optimization over a Hermitian matrix. It is attained by projecting onto the positive eigenspace of $\bar{\rho}_{\boldsymbol{\sigma}}$. Hence the supremum is bounded by $\|\bar{\rho}_{\boldsymbol{\sigma}}\|_1$. Using
\begin{equation}
    \|\bar{\rho}_{\boldsymbol{\sigma}}\|_1 \leq \sqrt{d}\,\|\bar{\rho}_{\boldsymbol{\sigma}}\|_F
\end{equation}
and the sign cancellation in the Rademacher average, together with $\|\rho_i\|_F\leq 1$, gives
\begin{equation}
    \mathbb{E}_{\boldsymbol{\sigma}}\!\left[\|\bar{\rho}_{\boldsymbol{\sigma}}\|_F\right]\leq \frac{1}{\sqrt{n}}
\end{equation}
This yields the stated bound; the full argument is given in Appendix~\ref{app:prop-idealrad}. \hfill$\square$

{\medskip\noindent\textit{Remark (dependence on the loss).}
Proposition~\ref{prop:idealrad} bounds the Rademacher complexity of the
\emph{score} class $\{\rho\mapsto\Tr(M\rho):M\in\POVM\}$, which does not
involve a loss; hence it is independent of $L$. The loss enters only in
Theorem~\ref{thm:main}, where the contraction lemma scales this complexity
by $L$.}

Combining this capacity bound with Hoeffding concentration for the shot estimates gives the main theorem, proved in Appendix~\ref{app:thm-main}.

\begin{theorem}[Generalization bound with finite shots]\label{thm:main}
Let $M \in \POVM$ be chosen independently of the shot outcomes used to compute $\hat{R}_{n,S,\ell}(M)$, and let $\ell:[0,1]\times\{0,1\}\to[0,1]$ be $L$-Lipschitz in its first argument. Then, with probability at least $1-\delta$ over the sample draw and the measurement outcomes,
\begin{equation}
  R_\ell(M)
  \leq
  \hat{R}_{n,S,\ell}(M)
  + 2L\sqrt{\frac{d}{n}}
  + L\sqrt{\frac{\log(4n/\delta)}{2S}}
  + \sqrt{\frac{\log(2/\delta)}{2n}}
  \label{eq:main-bound}
\end{equation}
\end{theorem}

\paragraph{Scope of the finite-shot guarantee.}
Theorem~\ref{thm:main} assumes that $M$ is chosen independently of the measurement outcomes used to compute the observed empirical risk $\hat{R}_{n,S,\ell}(M)$. The theorem does not require $M$ to be independent of the training states and labels. The sample term follows from a Rademacher bound that holds uniformly over $\POVM$, so it applies to every $M$ simultaneously, including a hypothesis fitted to $\{(\rho_i,y_i)\}_{i=1}^n$. Independence is required only with respect to the shot outcomes $Z_{i,s}$. The guarantee therefore covers the following two-stage protocol: (1)~\emph{training}: run the learning algorithm on the training states, using exact expectation values or a separate measurement budget; (2)~\emph{evaluation}: measure each training state $S$ additional times with the returned $M$ and report $\hat{R}_{n,S,\ell}(M)$. The budget $B=nS$ counts the shots of stage~(2). This setting covers, for example, the evaluation of a fixed trained hypothesis using a separate shot budget, as well as pre-experimental resource planning for a given measurement class. If the same shot-noisy losses are also used to select $M$, as in adaptive training, this independence no longer holds. In that case, additional uniform-convergence or algorithm-specific arguments are needed to control the dependence between the selected hypothesis and the measurement noise. Extending the analysis to fully adaptive shot-noisy training remains an open problem.

Table~\ref{tab:bound-comparison} contrasts the structure of the bound with
its classical analogue.

\begin{table}[h]
\centering
\caption{Structural comparison of classical and quantum generalization bounds.}
\label{tab:bound-comparison}
\setlength{\tabcolsep}{7pt}\renewcommand{\arraystretch}{1.3}
\begin{tabular}{@{}lll@{}}
\toprule
\textbf{Property} & \textbf{Classical} & \textbf{Quantum (this work)} \\
\midrule
Score observation       & Direct               & $S$ binary shots (Born rule) \\
Sample error term       & $O(1/\sqrt{n})$      & $O(\sqrt{d/n})$ \\
Shot error term         & None                 & $O(\sqrt{(\log n)/S})$ \\
Resources to plan       & $n$ only             & $n$ \emph{and} $S$ \\
Surrogate budget-rate scaling     & $O(B^{-1/2})$        & $O(B^{-1/4})$ \\
\bottomrule
\end{tabular}
\end{table}

\subsection{Conservative Surrogate Budget Allocation}

Substituting $S = B/n$ into~\eqref{eq:main-bound} and dropping the confidence term yields a U-shaped function of $n$:

\begin{equation}
  \mathcal{G}(n) =
  L\Biggl(
  \underbrace{\frac{2\sqrt{d}}{\sqrt{n}}}_{f_1(n)\;\downarrow}
  +
  \underbrace{\sqrt{\frac{n\log(4n/\delta)}{2B}}}_{f_2(n)\;\uparrow}
  \Biggr)
  \label{eq:tradeoff}
\end{equation}
Since $L$ multiplies both terms, it scales the bound but does not change the location of its minimum. Consequently, the allocation results below are independent of $L$; we report numerical values for $L=1$, the Lipschitz constant of the loss used in our experiments.

\begin{figure}[h]
\centering
% Fig: Shot-sample tradeoff (revised) -- B=1000, delta=0.05, L=1
% G(n) = 2 sqrt(d/n) + sqrt(n log(4n/delta) / (2B))  (consistent with Theorem main)
% Surrogate allocation n* = 2 sqrt(2dB / log(4B/delta)): 53 (d=4), 106 (d=16)
\begin{tikzpicture}
\begin{axis}[
  width=10cm, height=6.5cm,
  xlabel={Number of training states $n$},
  ylabel={Generalization bound $\mathcal{G}(n)$},
  xmin=8, xmax=320,
  ymin=0.35, ymax=3.1,
  xtick={10,40,80,160,320},
  xticklabels={10,40,80,160,320},
  xmode=log,
  legend pos=north east,
  legend style={font=\small},
  grid=major, grid style={gray!25},
  tick label style={font=\small},
  label style={font=\small},
  title style={font=\small},
  title={Shot-sample trade-off: $B=1000$, $\delta=0.05$},
]
\addplot[color=blue!70!black, thick, smooth] coordinates {
  (8,1.575)(10,1.448)(15,1.263)(20,1.166)(25,1.108)(30,1.072)(35,1.049)
  (40,1.034)(50,1.021)(53,1.020)(60,1.021)(70,1.028)(80,1.039)(100,1.070)
  (120,1.107)(140,1.146)(160,1.186)(200,1.267)(250,1.366)(300,1.461)
};
\addlegendentry{$d=4$ (2 qubits)}
\addplot[color=red!70!black, thick, smooth] coordinates {
  (8,2.989)(10,2.713)(15,2.296)(20,2.060)(25,1.908)(30,1.802)(35,1.725)
  (40,1.667)(50,1.587)(60,1.537)(70,1.506)(80,1.487)(100,1.470)(106,1.469)
  (120,1.472)(140,1.484)(160,1.502)(200,1.550)(250,1.619)(300,1.692)
};
\addlegendentry{$d=16$ (4 qubits)}
\addplot[color=blue!70!black, only marks, mark=*, mark size=4pt]
  coordinates {(53,1.020)};
\node[above, font=\footnotesize, text=blue!70!black]
  at (axis cs:53,1.020) {$n^*\!=\!53$};
\addplot[color=red!70!black, only marks, mark=square*, mark size=3.5pt]
  coordinates {(106,1.469)};
\node[above right, font=\footnotesize, text=red!70!black]
  at (axis cs:106,1.469) {$n^*\!=\!106$};
\addplot[color=blue!40, dashed, thin] coordinates {(53,0.35)(53,1.020)};
\addplot[color=red!40,  dashed, thin] coordinates {(106,0.35)(106,1.469)};
\node[font=\scriptsize, align=center, text=black!65]
  at (axis cs:11,0.8) {Sample-limited};
\node[font=\scriptsize, align=center, text=black!65]
  at (axis cs:200,2.1) {Shot-limited};
\draw[->, black!70, thin] (axis cs:17,0.8)  -- (axis cs:30,1.05);
\draw[->, black!70, thin] (axis cs:248,2.0) -- (axis cs:225,1.60);
\end{axis}
\end{tikzpicture}
\caption{Shot-sample trade-off for $B=1000$, $L=1$ and $\delta=0.05$. The dashed vertical lines show the surrogate allocations $\nstar=53$ for $d=4$ and $\nstar=106$ for $d=16$. The exact minimizers of the original logarithmic objective can be computed numerically; they are $55.2$ and $107.1$, and the bound values at the surrogate and exact minimizers agree to within $0.02\%$ (Table~\ref{tab:exact-vs-surrogate}). To the left of the allocation, the bound is mainly sample-limited, so increasing the number of training states reduces the bound. To the right, the bound becomes shot-limited, as each state receives too few measurements for accurate score estimation.
}
\label{fig:tradeoff}
\end{figure}

To obtain a closed-form allocation rule, we replace $\log(4n/\delta)$ by $\log(4B/\delta)$ in~\eqref{eq:tradeoff}, yielding a surrogate objective $\mathcal{G}_C(n) = L\bigl(2\sqrt{d}/\sqrt{n} + \sqrt{n\log(4B/\delta)/(2B)}\bigr)$ that has the same asymptotic scaling in $B$ as the original.  This substitution gives an upper bound on the original objective for $n \leq B$ (since $\log(4n/\delta) \leq \log(4B/\delta)$ whenever $n \leq B$), so the allocation derived from $\mathcal{G}_C$ remains a valid planning rule for the original bound.  The surrogate is used to obtain a closed-form rule; the exact minimizer of the original logarithmic objective satisfies the implicit equation
\[
  n\bigl(\log(4n/\delta) + 1\bigr) = 2\sqrt{2dB\log(4n/\delta)},
\]
which can be solved numerically.  The resulting closed-form surrogate allocation is stated in Appendix~\ref{app:thm-allocation}.

{\medskip\noindent\textit{Remark (logarithmic constant).} The factor $4$ in $\log(4n/\delta)$ arises from allocating failure probability $\delta/2$ to the shot event in Theorem~\ref{thm:main}. The simulations in Section~\ref{sec:experiments} use the allocation computed with $\log(2B/\delta)$ instead of $\log(4B/\delta)$, which differs from $\nstar$ in~\eqref{eq:opt-alloc} by about $3\%$ (e.g., $55$ vs.\ $53$ states at $B=1000$, $d=4$). All bound values reported for the experiments use the exact constants of Theorem~\ref{thm:main}.}

\begin{theorem}[Conservative surrogate allocation rule]\label{thm:allocation}
Under total budget $B = n \times S$, the allocation minimizing the
conservative surrogate $\mathcal{G}_C(n)$ is, for every Lipschitz constant $L>0$:
\begin{equation}
  \nstar = 2\sqrt{\frac{2dB}{\log(4B/\delta)}},
  \qquad
  \Sstar = \frac{B}{\nstar}
         = \frac{1}{2}\sqrt{\frac{B\log(4B/\delta)}{2d}}
  \label{eq:opt-alloc}
\end{equation}
$\nstar$ scales as $\sqrt{B/\log B}$ and $\Sstar$ as $\sqrt{B\log B}$.  At this surrogate
allocation:
\[
  \mathcal{G}_C(\nstar) = 2^{5/4}L\left(\frac{d\log(4B/\delta)}{B}\right)^{\!1/4}
  = O\!\left(\!L\left(\frac{d\log(4B/\delta)}{B}
  \right)^{\!1/4}\right)
\]
\end{theorem}

{
\paragraph{Closed-form versus numerical allocation.}
For fixed $(d,B,\delta)$, the exact minimizer of~\eqref{eq:tradeoff} can be computed by a one-dimensional search. Table~\ref{tab:exact-vs-surrogate} compares it with $\nstar$. The two allocations differ by at most $8\%$ in $n$, while the corresponding bound values differ by less than $0.1\%$ in all cases, since $\mathcal{G}$ is flat near its minimum. Hence the surrogate incurs a negligible loss in bound value, and the closed form is not intended as a more accurate substitute for the numerical solution. It shows how the optimal allocation depends on $d$, $B$, $\delta$, and $L$, which a numerical solution for a single configuration does not:
\begin{itemize}[leftmargin=*,nosep]
  \item $\nstar$ scales as $\sqrt{B/\log B}$ and $\Sstar$ as $\sqrt{B\log B}$, so additional
        budget should be split between more states and more shots rather than
        assigned to either one;
  \item $\nstar\propto\sqrt{d}$ and $\Sstar\propto1/\sqrt{d}$, so a larger
        Hilbert space shifts the optimal budget toward more training states;
  \item the allocation is independent of the Lipschitz constant $L$;
  \item the bound value at the allocation scales as $(d\log B/B)^{1/4}$, which
        gives the budget required to reach a target guarantee (e.g., a factor
        of $16$ to halve it);
  \item the ratio $n/\nstar$ maps a given configuration to one of the regimes
        in Table~\ref{tab:regimes} without additional computation.
\end{itemize}

\begin{table}[h]
\centering
\caption{Closed-form surrogate allocation $\nstar$ versus the exact numerical minimizer $n_{\mathrm{exact}}$ of~\eqref{eq:tradeoff}, for $L=1$ and $\delta=0.05$ (continuous values, before rounding). The last column is the relative excess of the bound at $\nstar$ over its minimum.}
\label{tab:exact-vs-surrogate}
\setlength{\tabcolsep}{6pt}\renewcommand{\arraystretch}{1.2}
\begin{tabular}{@{}rrrrrrr@{}}
\toprule
$B$ & $d$ & $\nstar$ & $n_{\mathrm{exact}}$ & $\mathcal{G}(\nstar)$ & $\mathcal{G}(n_{\mathrm{exact}})$ & Excess (\%) \\
\midrule
   500 &  4 &  38.9 &  39.6 & 1.2007 & 1.2006 & 0.01 \\
   500 & 16 &  77.7 &  76.8 & 1.7314 & 1.7314 & 0.00 \\
  1000 &  4 &  53.2 &  55.2 & 1.0199 & 1.0197 & 0.02 \\
  1000 & 16 & 106.5 & 107.1 & 1.4694 & 1.4694 & 0.00 \\
  2000 &  4 &  73.1 &  76.8 & 0.8660 & 0.8657 & 0.03 \\
  2000 & 16 & 146.2 & 149.2 & 1.2468 & 1.2467 & 0.01 \\
  5000 &  4 & 111.4 & 119.1 & 0.6973 & 0.6969 & 0.06 \\
  5000 & 16 & 222.7 & 231.6 & 1.0030 & 1.0027 & 0.02 \\
 10000 &  4 & 153.4 & 166.1 & 0.5917 & 0.5912 & 0.09 \\
 10000 & 16 & 306.9 & 323.1 & 0.8505 & 0.8502 & 0.04 \\
\bottomrule
\end{tabular}
\end{table}
}

Table~\ref{tab:regimes} summarizes the three operating regimes and the recommended actions; Table~\ref{tab:opt-alloc} provides numerical values for representative budgets.

\begin{table}[h]
\centering
\caption{Operating regimes under fixed budget $B = n \times S$.}
\label{tab:regimes}
\setlength{\tabcolsep}{7pt}\renewcommand{\arraystretch}{1.3}
\begin{tabular}{@{}llll@{}}
\toprule
\textbf{Regime} & \textbf{Condition} & \textbf{Dominant term} & \textbf{Recommended action} \\
\midrule
Sample-limited
  & $n \ll \nstar$
  & $2\sqrt{d/n}$
  & Increase $n$, reduce $S$ \\
Balanced
  & $n \approx \nstar$
  & Both comparable
  & Near-optimal; increase $B$ if further reduction needed \\
Shot-limited
  & $n \gg \nstar$
  & $\sqrt{n\log n/B}$
  & Decrease $n$, increase $S$ \\
\bottomrule
\end{tabular}
\end{table}

\begin{table}[h]
\centering
\caption{Closed-form surrogate allocation and original logarithmic bound values evaluated at the surrogate allocation $(\nstar,\Sstar)$, with $\delta=0.05$ and $L=1$. Here $f_1(\nstar)$ is the sample term and $f_2(\nstar)$ is the original logarithmic shot term using $\log(4\nstar/\delta)$, not the surrogate $\log(4B/\delta)$. ``Dominant'' is $f_1+f_2$; ``Full'' adds the confidence term $\sqrt{\log(2/\delta)/(2\nstar)}$ and is the complete bound of Theorem~\ref{thm:main}. Since the loss takes values in $[0,1]$, the largest possible generalization gap is $1$; full-bound values above $1$ (marked $^{*}$) are vacuous.}
\label{tab:opt-alloc}
\begin{threeparttable}
\setlength{\tabcolsep}{5pt}\renewcommand{\arraystretch}{1.25}
{%
\begin{tabular}{@{}rrrrrrrr@{}}
\toprule
$B$ & $d$ & $\nstar$ & $\Sstar$
    & $f_1(\nstar)$ & $f_2(\nstar)$ & Dominant & \textbf{Full} \\
\midrule
\rowcolor{lightblue!40}
   500 &  4 &  39 &  12 & 0.641 & 0.579 & 1.220 & 1.437$^{*}$ \\
   500 & 16 &  78 &   6 & 0.906 & 0.853 & 1.759 & 1.913$^{*}$ \\
\rowcolor{lightblue!40}
  1000 &  4 &  53 &  18 & 0.549 & 0.482 & 1.031 & 1.218$^{*}$ \\
  1000 & 16 & 106 &   9 & 0.777 & 0.709 & 1.486 & 1.618$^{*}$ \\
\rowcolor{lightblue!40}
  2000 &  4 &  73 &  27 & 0.468 & 0.401 & 0.869 & 1.028$^{*}$ \\
  2000 & 16 & 146 &  13 & 0.662 & 0.600 & 1.262 & 1.375$^{*}$ \\
\rowcolor{lightblue!40}
  5000 &  4 & 111 &  45 & 0.380 & 0.318 & 0.697 & 0.826\tnote{$\dagger$} \\
  5000 & 16 & 223 &  22 & 0.536 & 0.472 & 1.007 & 1.098$^{*}$\tnote{$\dagger$} \\
\rowcolor{lightblue!40}
 10000 &  4 & 153 &  65 & 0.323 & 0.269 & 0.592 & 0.702\tnote{$\dagger$} \\
 10000 & 16 & 307 &  32 & 0.457 & 0.397 & 0.854 & 0.932\tnote{$\dagger$} \\
\bottomrule
\end{tabular}}
\begin{tablenotes}\footnotesize
  \item[$\dagger$] Theory-only rows.  Empirical consistency checks cover
  $B\in\{500,1000,2000\}$; the $B\in\{5000,10000\}$ entries illustrate the asymptotic
  $B^{-1/4}$ trend and were not included in the simulator experiments.
  $\nstar$ is rounded to the nearest integer and $\Sstar=\lfloor B/\nstar\rfloor$.
\end{tablenotes}
\end{threeparttable}
\end{table}

% ============================================================
\section{Experimental Consistency Checks}
\label{sec:experiments}
% ============================================================

\subsection{Objectives and Setup}

We run three kinds of experiments.  First, a \textbf{direct numerical check} (Experiment~0) computes the empirical Rademacher complexity of the full $\POVM$ class using the spectral formula of Proposition~\ref{prop:idealrad}, verifying that empirical estimates remain below the theoretical bound.  Second, \textbf{end-to-end VQC experiments} (Experiments~1--4) train variational quantum circuits and check that the empirical one-sided test-train gap remains below the theoretical bound. Third, Experiment~5 reuses the trained circuits to test whether the U-shaped trade-off appears in the empirical gap.  As VQCs constitute a restricted subclass of $\POVM$ on structured synthetic data, these experiments check that the bound holds; they are not designed to test its tightness.  The bound covers the worst case over all of $\POVM$; a VQC subclass on structured data is far from the worst case, so the observed gaps are expected to lie well below the bound.

\medskip\noindent\textbf{Implementation.}
All experiments are implemented using PennyLane~\cite{bergholm2018pennylane} with the \texttt{lightning.qubit} simulator.  VQCs employ $R_Y$ angle embedding and \texttt{BasicEntanglerLayers} with two layers. {The two input features are encoded by $R_Y$ rotations on the first two qubits, and the class-$1$ score is the probability of outcome $1$ on qubit~$0$, i.e., $M(\boldsymbol{\theta})=U(\boldsymbol{\theta})^\dagger (\lvert1\rangle\langle1\rvert\otimes I)U(\boldsymbol{\theta})$. All experiments follow the two-stage protocol described after Theorem~\ref{thm:main}. \emph{Training.} We optimize $\boldsymbol{\theta}$ with gradient descent (learning rate $0.05$, $10$ iterations, uniform random initialization in $[-\pi,\pi]$) on a squared-hinge surrogate loss, using exact (noiseless) simulator probabilities; no measurement shots are used during training. \emph{Evaluation.} We compute the observed empirical risk $\hat{R}_{n,S,\ell}(M)$ from $S$ fresh shots on each of the $n$ training states, so the shot outcomes are independent of $M$, as required by the theorem. We approximate the population risk $R_\ell(M)$ by the exact score risk on a held-out test set of $n_{\text{test}}=500$ states. All results are averaged over multiple random seeds.}

The evaluation loss is the absolute score loss
$\ell(p,y)=|p-y|$, which on $[0,1]$ coincides with the rescaled binary hinge loss
$$
\ell_{\mathrm{hinge}}(p,y)
=
\frac{1}{2}\max\{0,1-(2y-1)(2p-1)\}\;=\;|p-y|,
$$
which is bounded in $[0,1]$ and is $1$-Lipschitz in $p$; all bound values below use $L=1$.

We simulate finite shots because expectation-value based QNNs can have large output variance under finite sampling, even before hardware noise is included~\cite{DUTTA2025113058}.

\medskip\noindent\textbf{Benchmark datasets.}
Nine two-dimensional synthetic binary classification benchmarks are used: three two-moons variants (\texttt{moons\_noise005/015/030}), two concentric-circles variants (\texttt{circles\_factor03/06}), two Gaussian benchmarks (\texttt{classification\_easy/hard}), and two blob benchmarks (\texttt{blobs\_tight/spread}).

\medskip\noindent\textbf{System configurations.}
Experiments are conducted for $k=2$ qubits ($d=4$) and $k=4$ qubits ($d=16$).

\medskip\noindent\textbf{Validity metric.}
Theorem~\ref{thm:main} is a one-sided upper bound on
$R_\ell(M) - \hat{R}_{n,S,\ell}(M)$.  The appropriate empirical diagnostic
is the one-sided gap:
$$
  \widehat{\Delta}_{+}(n,S)
  = \max\{0,\; \hat{R}_\text{test} - \hat{R}_\text{train}\}
$$
Because the loss takes values in $[0,1]$, $\widehat{\Delta}_{+}\le1$ always; a bound value above $1$ is therefore vacuous, and such values are marked with $^{*}$.

We check whether $\widehat{\Delta}_+(n,S)$ remains below the bound of Theorem~\ref{thm:main}; it does in every tested configuration.  For the empirical bound diagnostic, the trained circuit is treated as fixed and the reported training risk is computed from fresh measurement samples that are not used during optimization, and the test risk from exact probabilities on a held-out test set.  This matches the two-stage protocol covered by Theorem~\ref{thm:main}.

\subsection{Experiment 0: Direct Rademacher Complexity Check}

For $d \in \{4, 8, 16\}$ and $n \in \{10,20,40,80,160,320\}$, the empirical Rademacher complexity is computed via the spectral formula: for each random sign vector $\boldsymbol{\sigma}$, the supremum is attained by summing the positive eigenvalues of the signed-average matrix $A_\sigma = \frac{1}{n}\sum_i \sigma_i \rho_i$.

\begin{figure}[h]
\centering
\includegraphics[width=0.82\textwidth]{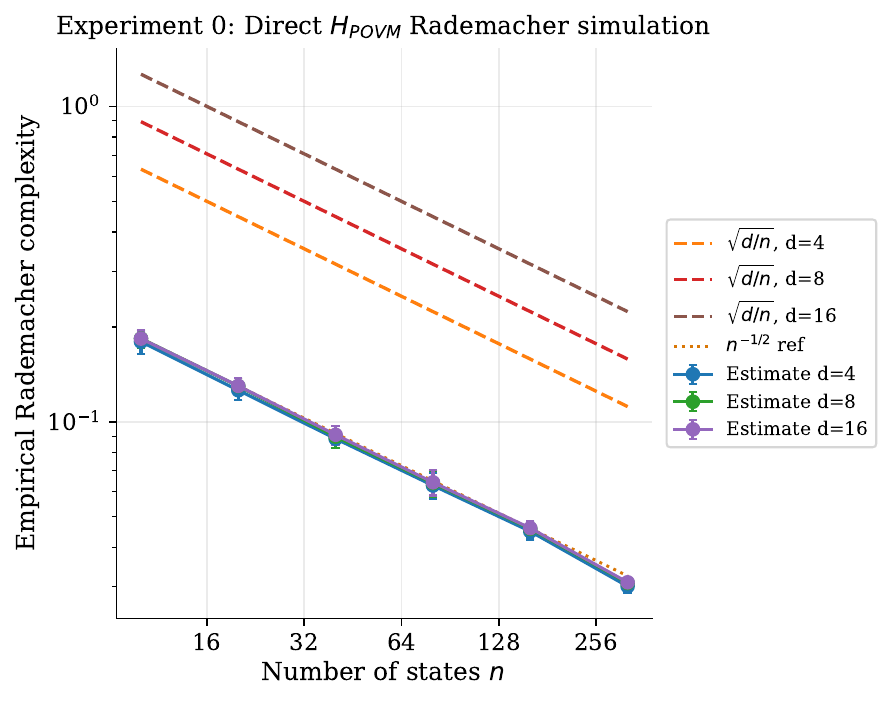}
\caption{\textbf{Experiment~0.}  Empirical Rademacher complexity of the full $\POVM$ class as a function of $n$.  Dashed curves: $\sqrt{d/n}$ upper bound.  Solid curves: empirical estimates (mean $\pm$ standard deviation over random sign vectors).  All estimates remain below the theoretical bound and follow the predicted $n^{-1/2}$ decay.}
\label{fig:exp0}
\end{figure}

\paragraph{Finding.}
The empirical estimates remain below the theoretical bound $\sqrt{d/n}$ for all tested $(d,n)$ configurations.  This provides a direct numerical consistency check of Proposition~\ref{prop:idealrad}.

\subsection{Experiment 1: Varying Number of Training States}

The number of shots is fixed at $S=100$.  The number of training states $n$
is swept over $\{10,20,40,80,160\}$ for $d \in \{4,16\}$ across all nine
benchmark datasets.

\begin{figure}[H]
\centering
\includegraphics[width=0.92\textwidth]{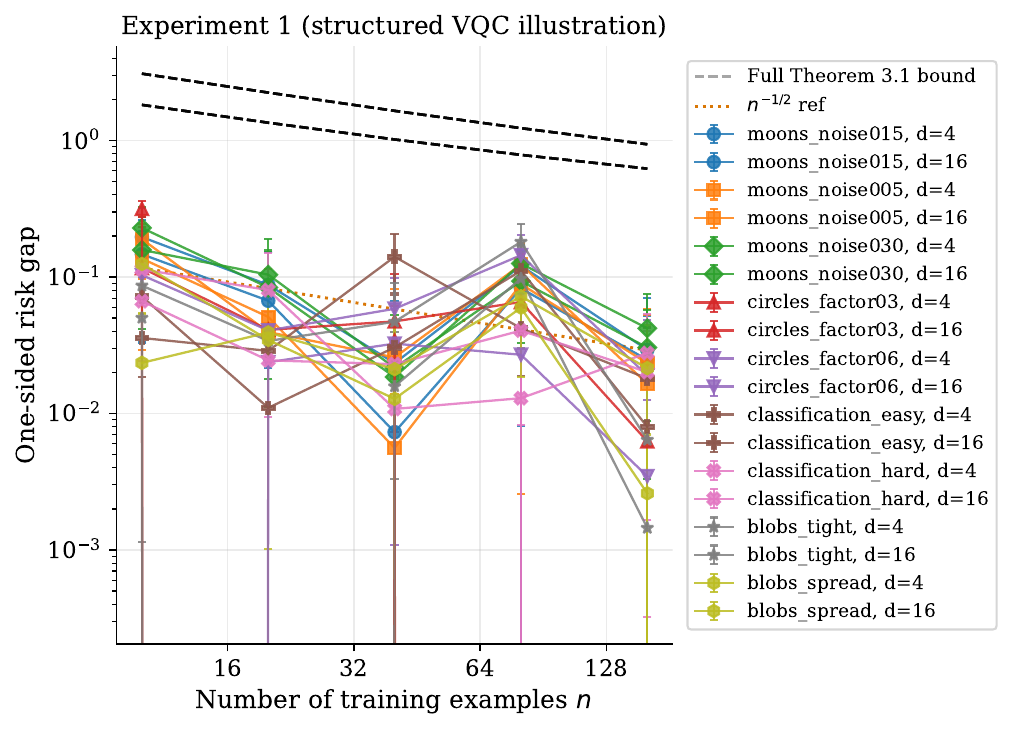}
\caption{\textbf{Experiment~1.}  One-sided empirical generalization gap as a function of $n$.  Black dashed line: Theorem~\ref{thm:main} bound. Dotted line: $n^{-1/2}$ reference slope.  All empirical gaps remain below the theoretical bound across all benchmarks and system sizes.}
\label{fig:exp1}
\end{figure}

\begin{table}[H]
\centering
\caption{Experiment~1: aggregate log-log slope of the one-sided risk gap as a function of $n$.  Reference slope predicted by theory: $-1/2$.}
\label{tab:exp1-slopes}
\setlength{\tabcolsep}{8pt}\renewcommand{\arraystretch}{1.25}
\begin{tabular}{@{}rrrr@{}}
\toprule
$d$ & Datasets & Mean slope & Std.\ (across datasets) \\
\midrule
\rowcolor{lightgray}
 $4$ & $9$ & $-0.607$ & $0.371$ \\
$16$ & $9$ & $-1.216$ & $1.365$ \\
\bottomrule
\multicolumn{4}{l}{\small Theoretical reference slope: $-0.5$.}
\end{tabular}
\end{table}

\paragraph{Finding.}
The bound holds in every configuration.  For $d=4$, the empirical log-log slope ($-0.607$) is close to the theoretical prediction of $-0.5$.  For $d=16$, the slope exhibits greater variance, as empirical gaps are near-zero and VQC optimization variance is the dominant contributor; this case provides a qualitative consistency check only.

\subsection{Experiment 2: Varying Shots Per Training State}

The number of training states is fixed at $n=80$.  The shots per state $S$ is swept over $\{10,25,50,100,250,500\}$ for $d \in \{4,16\}$.

\begin{figure}[h]
\centering
\includegraphics[width=0.92\textwidth]{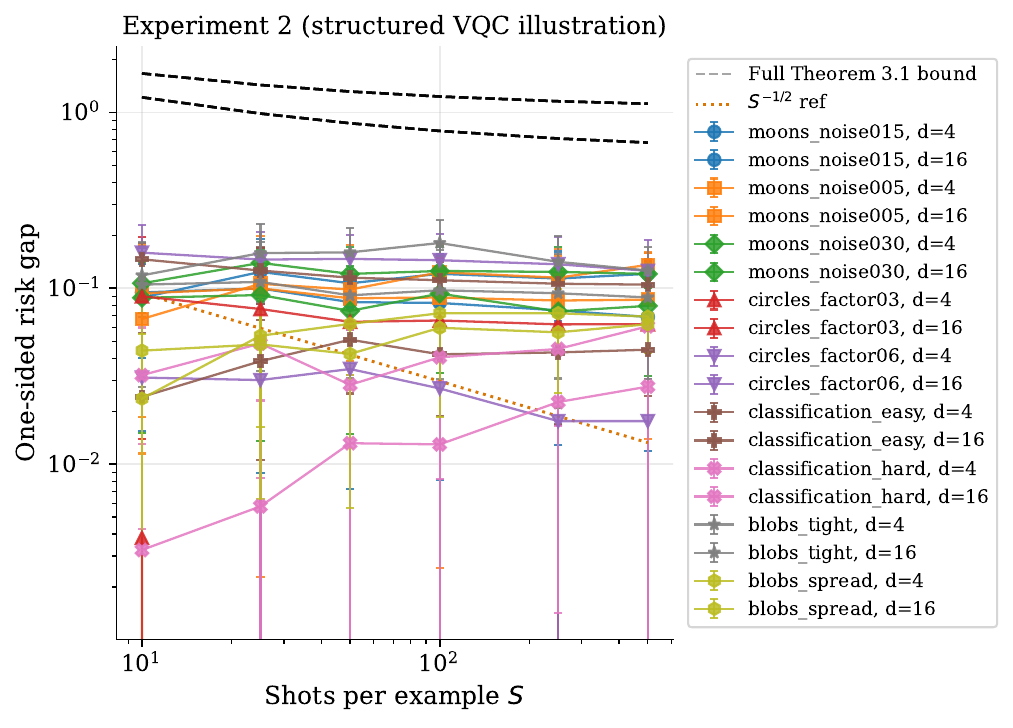}
\caption{\textbf{Experiment~2.}  One-sided empirical gap as a function of $S$.  The theoretical shot term decays as $S^{-1/2}$.  Empirical VQC gaps are approximately flat. For the absolute score loss, shot noise is unbiased and therefore does not change the observed training risk in expectation (see the Finding below).}
\label{fig:exp2}
\end{figure}

\paragraph{Finding.}
The bound holds for every tested value of $S$.  The approximate flatness of empirical gaps, rather than the $S^{-1/2}$ decay of the shot term, follows from the choice of loss. The absolute score loss $\ell(p,y)=|p-y|$ is affine in $p$ on $[0,1]$ for each label, so $\mathbb{E}[\ell(\hat p_{M,S}(\rho_i),y_i)]=\ell(p_M(\rho_i),y_i)$; that is, finite-shot noise is \emph{unbiased} for this loss. Consequently, the shot error of the observed training risk is a zero-mean average over all $nS$ shots with standard deviation at most $1/(2\sqrt{nS})$, and the expected gap is independent of $S$. The shot term of Theorem~\ref{thm:main} is a worst-case term. It holds for every Lipschitz loss and bounds each state separately, so it is not tight for this loss. For losses that are not affine in the score, shot noise biases the observed risk and the shot-limited regime becomes observable (Experiment~5).

\subsection{Experiment 3: Fixed Total Budget}

The total budget $B$ is fixed at each value in $\{500, 1000, 2000\}$.
For each $(B,d)$ pair, the number of training states $n$ is swept on a
logarithmic grid with $S = \lfloor B/n \rfloor \geq 5$.

\begin{figure*}[t!]
\centering
\begin{subfigure}[t]{0.48\textwidth}
  \includegraphics[width=\textwidth]{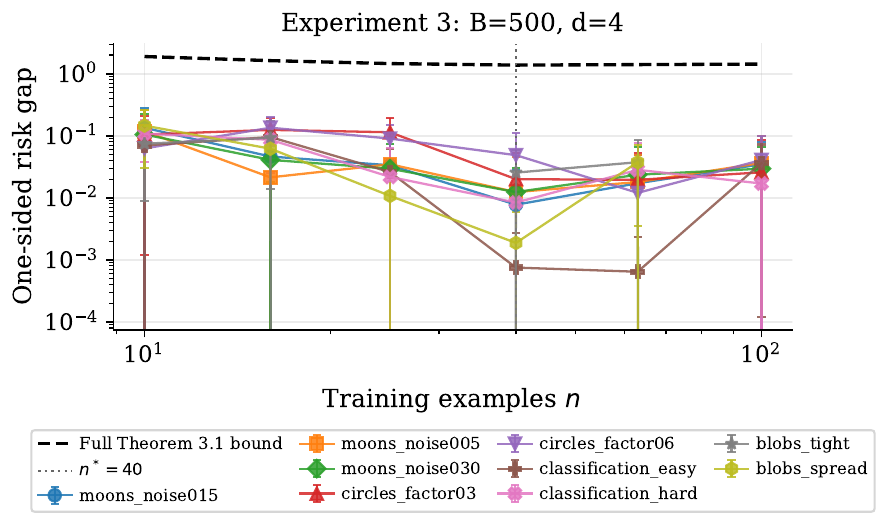}
  \caption{$B=500$, $d=4$}
\end{subfigure}\hfill
\begin{subfigure}[t]{0.48\textwidth}
  \includegraphics[width=\textwidth]{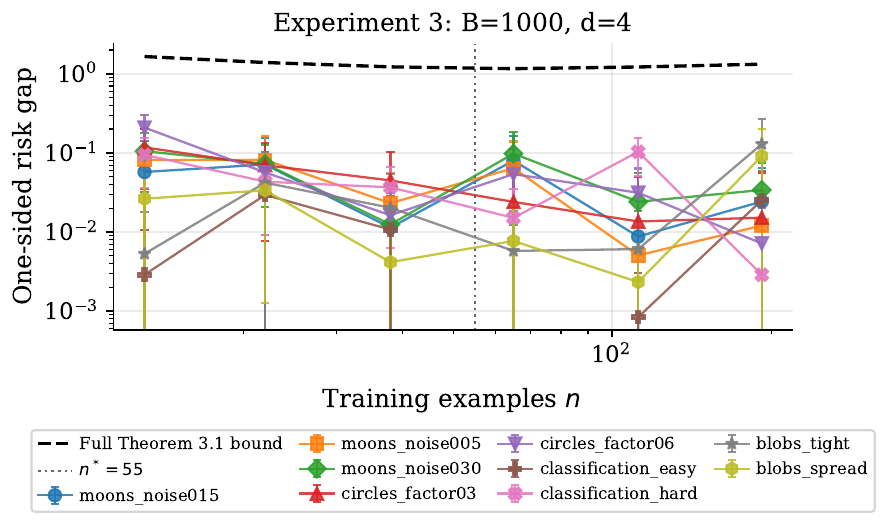}
  \caption{$B=1000$, $d=4$}
\end{subfigure}

\vspace{0.5em}

\begin{subfigure}[t]{0.48\textwidth}
  \includegraphics[width=\textwidth]{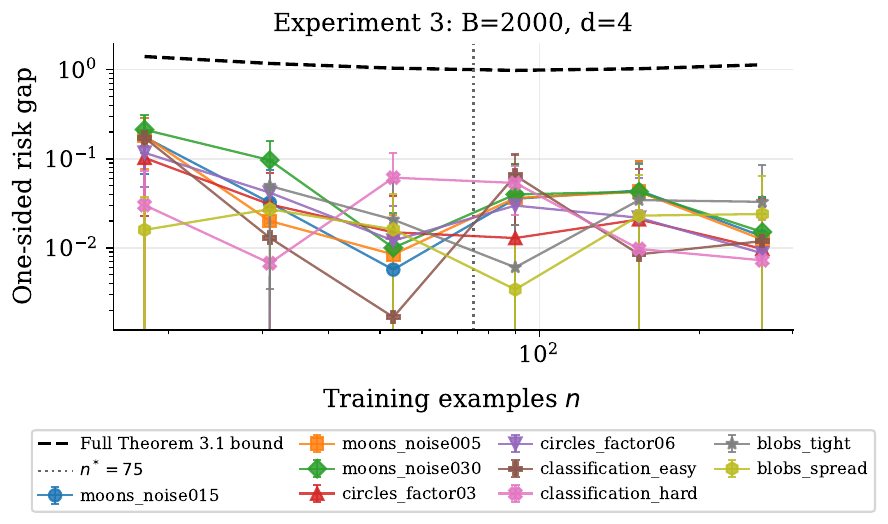}
  \caption{$B=2000$, $d=4$}
\end{subfigure}\hfill
\begin{subfigure}[t]{0.48\textwidth}
  \includegraphics[width=\textwidth]{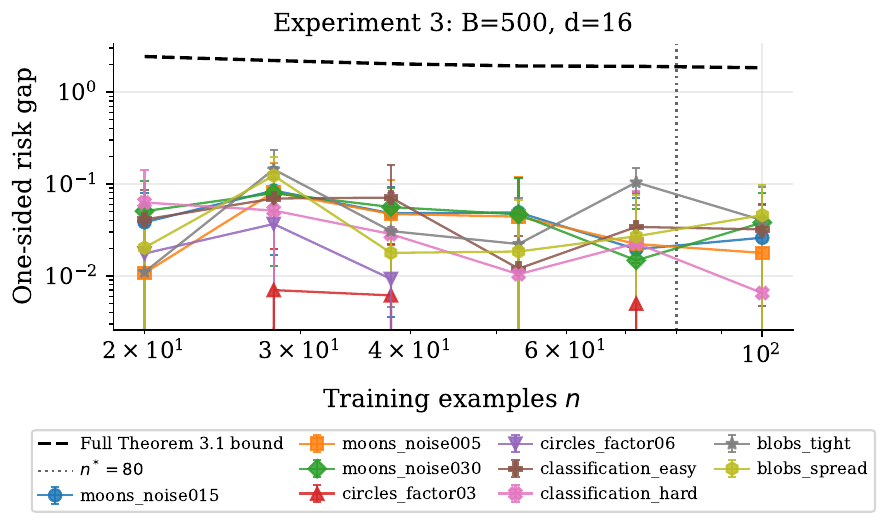}
  \caption{$B=500$, $d=16$}
\end{subfigure}

\vspace{0.5em}

\begin{subfigure}[t]{0.48\textwidth}
  \includegraphics[width=\textwidth]{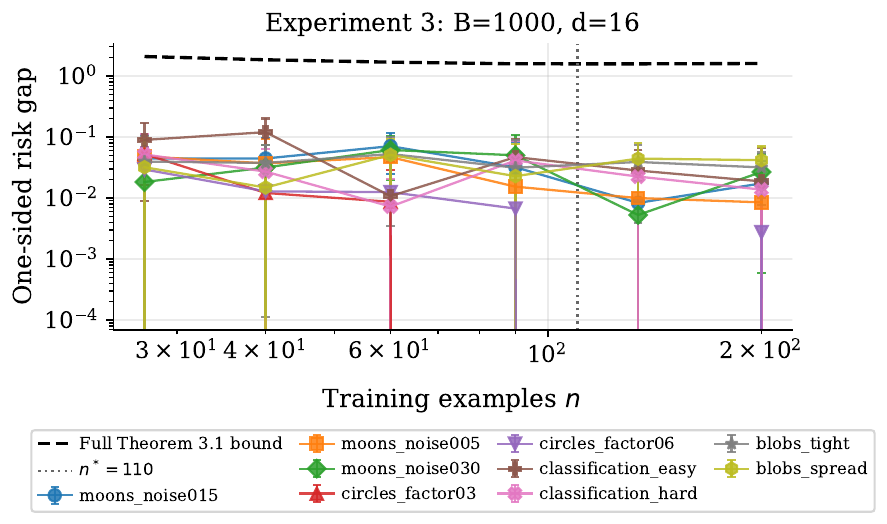}
  \caption{$B=1000$, $d=16$}
\end{subfigure}\hfill
\begin{subfigure}[t]{0.48\textwidth}
  \includegraphics[width=\textwidth]{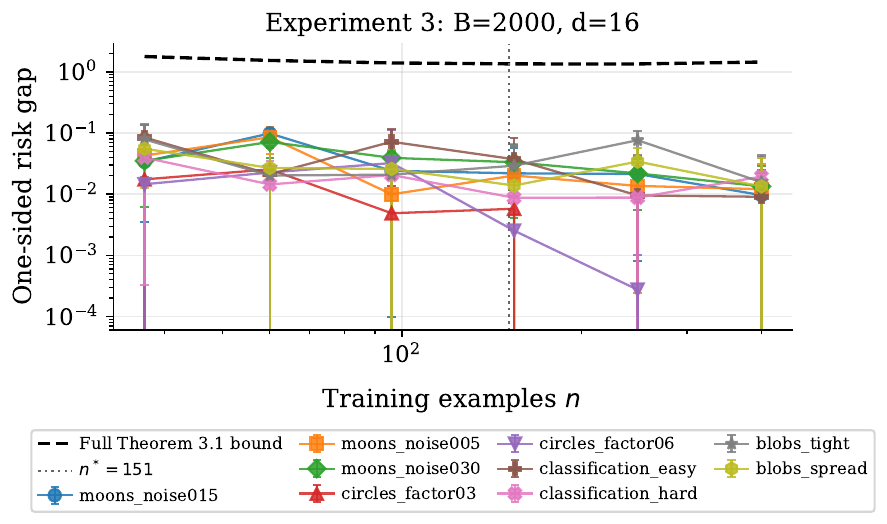}
  \caption{$B=2000$, $d=16$}
\end{subfigure}

\caption{\textbf{Experiment~3.}  One-sided empirical gap under fixed budget $B = nS$, as a function of $n$.  Colored curves: empirical VQC gaps across nine benchmark datasets.  Black dashed line: Theorem~\ref{thm:main} bound. Dotted vertical line: closed-form surrogate allocation $\nstar$ illustrating the theoretical planning rule rather than a confirmed empirical minimum (computed with $\log(2B/\delta)$ in place of $\log(4B/\delta)$; the value from~\eqref{eq:opt-alloc} is about $3\%$ smaller, see Section~\ref{sec:system}).  The U-shaped trade-off is visible in the theoretical bound; empirical VQC gaps do not, because for the absolute score loss the expected gap depends on $n$ only through the sample term (see the Finding of Experiment~2). Experiment~5 shows the empirical U-shape for a concave score loss.}
\label{fig:exp3}
\end{figure*}

\paragraph{Finding.}
The bound holds in every configuration.  The U-shaped trade-off predicted by the analysis is visible in the bound curves. The empirical VQC gaps are too flat to precisely locate the empirical minimum. This is expected for the absolute score loss: because shot noise is unbiased for this loss, the expected gap at fixed $B$ depends on $n$ only through the sample term and cannot increase in the shot-limited regime. A U-shaped bound does not imply a U-shaped gap for every loss, since the bound holds uniformly over all Lipschitz losses. Experiment~5 evaluates a loss for which the shot-limited regime is observable.

\subsection{Experiment 4: Bound-to-Gap Ratios}

\begin{table}[h]
\centering
\caption{Bound-to-gap ratios at the simulated allocation $(n,S)$ closest to the predicted allocation $(\nstar,\Sstar)$, averaged over all nine benchmark datasets.  Entries with $\widehat{\Delta}_+ < 10^{-3}$ are excluded as numerically unreliable. Slack $=$ Bound $-$ Gap; the ratio is Bound $/$ Gap. All six bound values exceed the trivial value $1$ (marked $^{*}$) and are vacuous for the normalized loss; Table~\ref{tab:opt-alloc} shows that the bound becomes non-vacuous from $B\approx5000$ for $d=4$ and $B\approx10^4$ for $d=16$.}
\label{tab:tightness}
\setlength{\tabcolsep}{5pt}\renewcommand{\arraystretch}{1.25}
{%
\begin{tabular}{@{}rrrrrlll@{}}
\toprule
$B$ & $d$ & $n$ & $S$ & Bound & Gap & Slack & Ratio / reported \\
\midrule
\rowcolor{lightgray}
 500 &  4 &  40 & 12 & 1.427$^{*}$ & 0.0203 & 1.4068 & $94.15$ / $9/9$ \\
 500 & 16 &  80 &  6 & 1.901$^{*}$ & 0.0102 & 1.8907 & $268.40$ / $9/9$ \\
\rowcolor{lightgray}
1000 &  4 &  55 & 18 & 1.205$^{*}$ & 0.0162 & 1.1890 & $128.44$ / $9/9$ \\
1000 & 16 & 110 &  9 & 1.603$^{*}$ & 0.0061 & 1.5966 & $329.41$ / $8/9$ \\
\rowcolor{lightgray}
2000 &  4 &  75 & 26 & 1.028$^{*}$ & 0.0096 & 1.0181 & $181.22$ / $9/9$ \\
2000 & 16 & 151 & 13 & 1.363$^{*}$ & 0.0047 & 1.3581 & $463.29$ / $9/9$ \\
\bottomrule
\end{tabular}}
\end{table}

\paragraph{Finding.}
Bound-to-gap ratios range from $94\times$ to $463\times$.  The bound is a worst-case guarantee over all of $\POVM$, whereas the trained circuits form a restricted subclass on structured benchmarks. At the simulated budgets $B\le2000$, every bound value also exceeds the trivial value $1$, so the absence of violations says nothing about the practical utility of the bound; the bound becomes informative only at larger budgets (Table~\ref{tab:opt-alloc}). Section~\ref{sec:conservatism} analyzes the sources of its conservatism.

{
\subsection{Experiment 5: Empirical U-Shape for a Concave Score Loss}
\label{sec:exp5}

Experiments~2 and~3 use the absolute score loss, for which finite-shot noise is unbiased, so the shot-limited regime of the bound cannot appear in the empirical gap. To test whether this regime is observable, we repeat the fixed-budget experiment with the bounded, concave score loss
\[
  \ell_{\mathrm{cq}}(p,y) = 1-\bigl(1-|p-y|\bigr)^2,
  \qquad\text{i.e.}\qquad
  \ell_{\mathrm{cq}}(p,1)=1-p^2,\quad \ell_{\mathrm{cq}}(p,0)=2p-p^2,
\]
which takes values in $[0,1]$ and is $2$-Lipschitz in $p$, so Theorem~\ref{thm:main} applies with $L=2$. Since
$\mathbb{E}[\hat p_{M,S}(\rho_i)^2]=p_M(\rho_i)^2+p_M(\rho_i)(1-p_M(\rho_i))/S$,
\begin{equation}
  \mathbb{E}\bigl[R_\ell(M)-\hat{R}_{n,S,\ell}(M)\bigr]
  = \underbrace{R_\ell(M)-\hat{R}_{n,\ell}(M)}_{\text{sample part}}
  + \underbrace{\frac{n}{B}\cdot\frac{1}{n}\sum_{i=1}^n p_M(\rho_i)\bigl(1-p_M(\rho_i)\bigr)}_{\text{shot part}},
  \label{eq:concave-decomp}
\end{equation}
where the expectation is taken over the shot outcomes. Under $B=nS$, the shot part grows linearly in $n$ and the sample part decreases in $n$; hence the expected gap is U-shaped, matching the shape of the bound~\eqref{eq:tradeoff}.

\paragraph{Setup.}
We use the same datasets, circuits, and training procedure as in Experiments~1--4 (exact gradients, $10$ gradient-descent iterations, learning rate $0.05$), with $n\in\{5,10,20,40,80,160,320,500\}$, $B\in\{500,1000,2000\}$, $S=\lfloor B/n\rfloor\ge2$, $d\in\{4,16\}$, nine datasets, and six random seeds. Since training uses no shots, we train each circuit once per $(\text{dataset},d,\text{seed},n)$ and evaluate it for every $B$. We sample shot outcomes from the exact Born probabilities, $\hat p_{M,S}(\rho_i)\sim\mathrm{Binomial}(S,p_M(\rho_i))/S$, which is equivalent in distribution to a shot-based simulator for a single-qubit read-out, and we estimate expectations over shots from $200$ independent draws. Both losses are evaluated on the same trained circuits.

\begin{figure}[t]
\centering
\includegraphics[width=\textwidth]{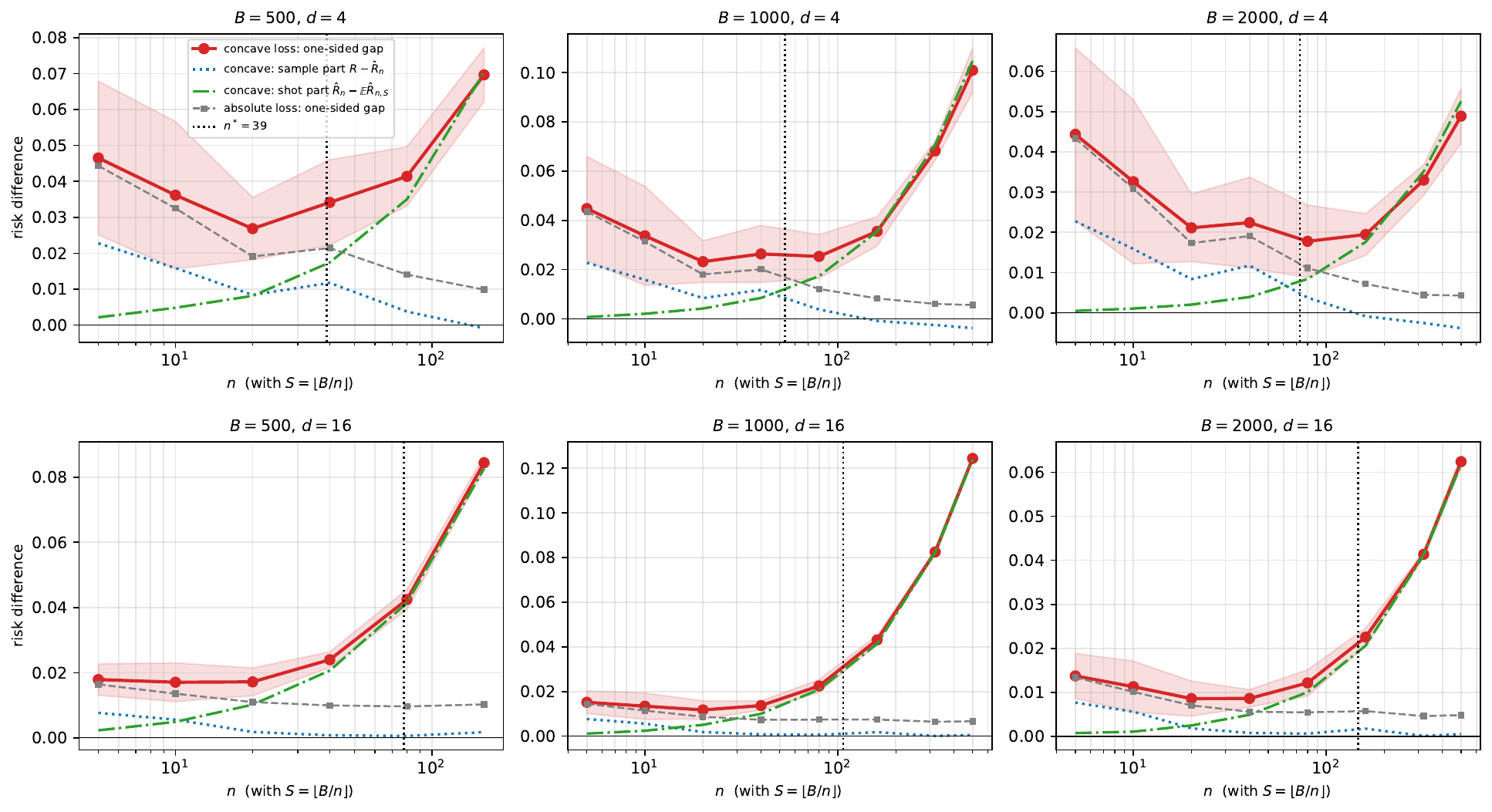}
\caption{\textbf{Experiment~5.} Fixed-budget sweep ($B=nS$) for the concave score loss $\ell_{\mathrm{cq}}$ (red: one-sided gap, mean $\pm$ standard deviation over nine datasets; blue dotted: sample part; green dash-dotted: shot part of~\eqref{eq:concave-decomp}) and for the absolute score loss on the same circuits (grey). Dotted vertical line: surrogate allocation $\nstar$. The bound of Theorem~\ref{thm:main} (with $L=2$) lies between $1.9$ and $8.1$ for these configurations and is not shown.}
\label{fig:exp5}
\end{figure}

\begin{table}[t]
\centering
\caption{Experiment~5: location of the minimum of the mean one-sided gap over the grid $n\in\{5,10,\dots,500\}$, compared with the surrogate allocation
$\nstar$.}
\label{tab:exp5}
\setlength{\tabcolsep}{6pt}\renewcommand{\arraystretch}{1.2}
\begin{tabular}{@{}rrrcrc@{}}
\toprule
$B$ & $d$ & $\nstar$ & \multicolumn{2}{c}{Concave loss} & Absolute loss \\
\cmidrule(lr){4-5}
 & & & $n_{\min}$ & gap at $n_{\min}$ & shape \\
\midrule
 500 &  4 &  39 & 20 & 0.027 & decreasing \\
1000 &  4 &  53 & 20 & 0.023 & decreasing \\
2000 &  4 &  73 & 80 & 0.018 & decreasing \\
 500 & 16 &  78 & 10 & 0.017 & decreasing, then flat \\
1000 & 16 & 106 & 20 & 0.012 & decreasing, then flat \\
2000 & 16 & 146 & 20 & 0.009 & decreasing, then flat \\
\bottomrule
\end{tabular}
\end{table}

\paragraph{Results.}
For the concave loss, the empirical gap is U-shaped under a fixed budget (Figure~\ref{fig:exp5}). The measured shot part agrees with the prediction $\overline{p(1-p)}/S$ of~\eqref{eq:concave-decomp} to within $6\times10^{-4}$ in every configuration and grows linearly in $n$, while the sample part decreases; their sum has an interior minimum in all six $(B,d)$ panels. The U-shape is clear for $d=4$; for $d=16$ the left branch is shallow (for $B=500$, the gap is flat for $n\le20$) because the sample part is small, which is consistent with the effective-dimension argument (C2) in Section~\ref{sec:conservatism}. On the same circuits, the absolute loss shows no shot effect (the measured shot part is below $6\times10^{-4}$ in absolute value) and its gap decreases or stays flat in $n$, as explained for Experiments~2 and~3.

\paragraph{Discussion.}
The empirical minimizers lie at $n_{\min}\in\{10,20,80\}$, to the left of $\nstar\in[39,146]$ (Table~\ref{tab:exp5}). This is expected: the sample term $2L\sqrt{d/n}$ of the bound overestimates the sample part of the trained circuits (sources C1 and C2 in Section~\ref{sec:conservatism}), while the shot term is of the correct order, so the minimum of the bound lies at a larger $n$ than that of the empirical gap. The bound therefore captures the structure of the trade-off but, being distribution-free, does not locate the empirical optimum for this model class.

\paragraph{Limitations.}
The trained circuits fit the data only weakly ($10$ training iterations), which keeps the sample part small, and the bound values for $L=2$ are vacuous at these budgets.
}

\FloatBarrier

% ============================================================
\section{Discussion}
\label{sec:discussion}
% ============================================================

\subsection{Interpretation of the $B^{-1/4}$ Rate}

The $B^{-1/4}$ rate at the surrogate allocation is slower than the classical $O(B^{-1/2})$ rate. This follows from splitting one budget between two resources. In classical learning, the budget counts training examples and there is one error term of order $O(1/\sqrt{n})$. Here the budget counts individual binary measurement outcomes and must be split between $n$ and $S$, and the two error terms balance at $O(B^{-1/4})$. If the budget is instead counted in experimental runs, where one run $(\rho_i, S$ shots$)$ corresponds to one classical example, the two rates are directly comparable. The slower rate therefore reflects the unit in which the budget is counted, not an inefficiency of quantum learning.

\subsection{Implications for Quantum-Technology Resource Planning}

The surrogate allocation formula $(\nstar, \Sstar)$ provides a statistical starting point for designing quantum learning experiments on near-term devices. 

\paragraph{How to apply the surrogate allocation rule.}
Given a total measurement budget $B$ (e.g., a fixed number of circuit executions on a cloud quantum platform), qubit count $k$, and confidence parameter $\delta$, compute:
\[
  \nstar = 2\sqrt{\frac{2 \cdot 2^k \cdot B}{\log(4B/\delta)}},
  \qquad
  \Sstar = \frac{B}{\nstar}
\]
Collect $\nstar$ distinct training states and devote $\Sstar$ shots to each. This allocation minimizes the surrogate planning bound derived from Theorem~\ref{thm:main}; the exact minimizer of the original objective can be found numerically.

\paragraph{Sample-limited regime.}
If the current experimental configuration has $n \ll \nstar$ (few training states relative to the budget), the dominant source of generalization error is the finite-sample term $2\sqrt{d/n}$.  Reallocating budget from shots to training states (reducing $S$ and increasing $n$) will improve the worst-case bound.  Studies that use large $S$ and small $n$ fall into this regime.

\paragraph{Shot-limited regime.}
If the current configuration has $n \gg \nstar$ (many training states relative to the budget), the dominant source of error is the shot term $\sqrt{n \log n / B}$.  Reallocating toward fewer training states with more shots each will improve the bound.

\paragraph{Relevance to cloud quantum devices.}
On cloud quantum platforms, circuit executions are metered, queued, and subject to device availability.  The allocation formula provides a basis for partitioning a fixed execution quota. 

{
\paragraph{Non-uniform cost model.}
The budget $B=nS$ assumes a uniform cost per shot, which is appropriate when circuit executions dominate the cost of an experiment. This assumption does not hold on all platforms. On cloud devices, each distinct circuit, and hence each distinct input state, incurs compilation, transpilation, and submission overhead, whereas additional shots of a loaded circuit are cheap. In quantum-sensing and quantum-data settings, each new input state may require a new physical experiment. A more general cost model is $C = n\,c_{\mathrm{prep}} + nS\,c_{\mathrm{shot}}$, where $c_{\mathrm{prep}}\ge0$ is the cost of preparing a state and $c_{\mathrm{shot}}>0$ is the cost of a single measurement. Our analysis corresponds to $c_{\mathrm{prep}}=0$. A positive preparation cost increases the marginal cost of each additional training state and therefore shifts the optimal allocation toward fewer states and more shots per state than $(\nstar,\Sstar)$. In this case, \eqref{eq:opt-alloc} should be interpreted as an upper bound on the number of distinct states worth preparing. Extending the closed-form analysis to this cost model is left for future work (OP9).}

{
\subsection{Sources of Conservatism}
\label{sec:conservatism}

The bound-to-gap ratios in Table~\ref{tab:tightness} are large, and the bound is vacuous at the simulated budgets. We identify five sources of this conservatism, listed in decreasing order of their effect in our experiments.

\paragraph{(C1) Worst case over the full hypothesis class.}
The sample term bounds the Rademacher complexity of the entire class $\POVM$, which has $d^2$ real parameters, whereas the trained models belong to $\mathcal{H}_{\mathrm{VQC}}$, which has $2k$ parameters in our experiments (two layers on $k$ qubits). Class-specific bounds, such as the covering-number bound $O(\sqrt{T/n})$ of Caro et al.~\cite{caro2022generalization} for $T$ trainable gates or the encoding-dependent bounds of~\cite{Caro2021encodingdependent,berberich2024robust}, would replace $\sqrt{d/n}$ with a much smaller quantity for shallow circuits. These bounds do not model finite shots, but they can be substituted for the sample term of Theorem~\ref{thm:main} without changing the shot term.

\paragraph{(C2) Ambient versus effective dimension.}
Step~3 of the proof uses $\|A\|_1\le\sqrt{d}\|A\|_F$, which is tight only when all $d$ singular values of $\bar\rho_{\boldsymbol\sigma}$ are equal. If the data states span a subspace of dimension $r<d$, then $\operatorname{rank}\bar\rho_{\boldsymbol\sigma}\le r$ and $\sqrt{d}$ can be replaced by $\sqrt{r}$. In our $d=16$ experiments, the two input features are encoded on two of the four qubits, so all data states lie in a $4$-dimensional subspace and the sample term is overestimated by a factor of $\sqrt{16/4}=2$. This is consistent with the observation that the empirical gaps for $d=16$ are not larger than those for $d=4$.

\paragraph{(C3) Union bound and worst-case variance in the shot term.}
The shot term controls the maximum per-state deviation $\max_i|p_M(\rho_i)-\hat p_{M,S}(\rho_i)|$ via a union bound over the $n$ states, which introduces the $\log n$ factor, and Hoeffding's inequality assumes the worst-case Bernoulli variance $1/4$. Scores close to $0$ or $1$, as produced by confident classifiers~\cite{recioarmengol2025singleshot}, have much smaller variance, and the per-state deviations partially cancel in the average that defines the empirical risk.

\paragraph{(C4) Distribution-free analysis.}
The bound holds for every data distribution, whereas the two-dimensional benchmarks used here are far from the worst case. Average-case analyses of sampling noise~\cite{2024arXiv241014654H} are expected to be closer to the observed errors, at the cost of distributional assumptions.

\paragraph{(C5) Constants.}
Symmetrization, the contraction lemma, and the split $\delta_1=\delta_2=\delta/2$ each contribute constant factors. They affect neither the scaling nor the allocation rule, but they increase the absolute bound value.

\medskip\noindent
Sources (C1)--(C4) can be reduced by class- or data-dependent refinements. None of these refinements changes the structure of the trade-off: under $B=nS$, the sample term decreases and the shot term increases in $n$. The closed-form allocation should therefore be interpreted as a statement about the structure of the trade-off rather than about the magnitude of the generalization gap.
}

\subsection{Limitations}

\paragraph{L1. Conservative bound.}
The bound is distribution-free and is loose for structured model classes and data; its sources of conservatism are analyzed in Section~\ref{sec:conservatism}.

\paragraph{L2. No hardware noise model.}
The current analysis isolates statistical shot noise and does not model hardware noise, readout error, gate error, decoherence, or calibration drift.  On physical NISQ devices, these mechanisms impose additional noise floors that are not captured by the shot term. They can differ qualitatively from shot noise and include noise-induced trainability problems in variational algorithms~\cite{wang2021noise,RevModPhys.94.015004}. The allocation $(\nstar,\Sstar)$ is a statistical planning rule and should be read as a lower bound on the required number of shots, to be adjusted upward using device-specific noise characterization.  Incorporating hardware noise into the allocation framework is needed before the formula can serve as a practical guide for physical quantum experiments; we leave this to future work.

\paragraph{L3. Fixed measurement operator.}
The theorem requires $M$ to be selected independently of the shot outcomes used to form $\hat{R}_{n,S,\ell}(M)$.  This covers resource planning and post-training evaluation but does not directly cover adaptive training procedures in which $M$ is iteratively updated using the same shot-noisy empirical losses.  Extending the result to fully adaptive, shot-noisy training algorithms requires additional uniform convergence or algorithm-specific arguments.

\paragraph{L4. Restricted hypothesis class in experiments.}
The VQC experiments test consistency of the bound for a restricted subclass of $\POVM$ on synthetic structured data.  They show that the bound holds but are not designed to test its tightness.  Experiments that probe the bound near its worst-case configurations, e.g., random measurements on high-entropy states, are left for future work.

\paragraph{L5. Non-adaptive measurement model.}
The analysis assumes measurement bases are selected independently of previous outcomes (i.i.d.\ shot model).  Adaptive measurement strategies, in which the measurement basis is updated based on previous outcomes, may reduce the shot term below the Hoeffding bound.  Whether adaptivity can recover the classical $B^{-1/2}$ rate, and whether the no-cloning theorem rules this out, is an open question.

\subsection{Open Problems}

\begin{description}[leftmargin=0pt]
  \item[OP1.] Establish minimax lower bounds for $\hat{\mathcal{R}}_n^{(\infty)}(\POVM)$ using quantum
              Fano-type arguments; determine whether $\sqrt{d/n}$ is tight.
  \item[OP2.] Derive Rademacher complexity bounds for PQC-induced hypothesis
              classes via the Lipschitz structure of the parameter-shift rule,
              targeting bounds of order $O(\Lambda L/\sqrt{n})$ that are
              independent of $d$.
  \item[OP3.] Extend the allocation analysis to incorporate hardware noise
              models (decoherence, readout error) and determine how the
              surrogate planning allocation shifts in the presence of a
              noise floor.
  \item[OP4.] Determine whether adaptive measurement strategies can improve
              the $B^{-1/4}$ rate, and whether the no-cloning theorem
              imposes an information-theoretic lower bound.
  \item[OP5.] Extend the result to adaptive, shot-noisy training procedures
              using uniform convergence or algorithm-dependent arguments.
  \item[OP6.] Extend the bounds to quantum covariate shift under
              channel-induced state drift.
  \item[OP7.] Generalize the Rademacher analysis to multi-outcome POVMs and
              observable regression tasks.
  \item[OP8.] Formalize the connection between the signed-average matrix
              $\bar{\rho}_{\boldsymbol{\sigma}}$ and classical shadow
              tomography estimators.
  \item[OP9.] Extend the allocation rule to the weighted cost
              model $C=n\,c_{\mathrm{prep}}+nS\,c_{\mathrm{shot}}$, in which
              preparing a new state has a non-negligible cost.
\end{description}

% ============================================================
\section{Conclusion}
\label{sec:conclusion}
% ============================================================

We studied how a finite measurement budget $B$ on a quantum device or simulator should be split between the number of distinct training states $n$ and the number of shots $S$ per state.

The main theoretical result is a distribution-free generalization bound for quantum classifiers (Theorem~\ref{thm:main}) that accounts for both $n$ and $S$.  The bound decomposes into a sample term $O(\sqrt{d/n})$ and a shot term $O(\sqrt{(\log n)/S})$.  Minimizing a closed-form surrogate of the bound over $n$ under the constraint $B = nS$ yields the allocation rule $\nstar = 2\sqrt{2dB/\log(4B/\delta)}$, $\Sstar = B/\nstar$, with the worst-case surrogate bound decaying at rate $O(B^{-1/4})$ at the surrogate optimum.

The bound is conservative, as is expected for a distribution-free worst-case result covering the full class $\POVM$.  Empirical consistency checks using PennyLane on 2-qubit and 4-qubit variational quantum circuits across nine synthetic classification benchmarks show that all one-sided empirical generalization gaps remain below the theoretical bound in every tested configuration of $n$, $S$, and $B$. At the simulated budgets, the bound exceeds the trivial value one, so these experiments establish consistency rather than practical tightness. The bound becomes non-vacuous for budgets of approximately $5000$ to $10^4$ shots, and Section~\ref{sec:conservatism} identifies five sources of its conservatism. For a concave score loss, the empirical gap is U-shaped under a fixed budget, and the measured shot contribution matches its analytical value.

A \emph{sample-limited} configuration (few training states, many shots per state) should reallocate budget toward additional training states.  A \emph{shot-limited} configuration (many states, few shots each) should do the opposite.  The surrogate allocation prescribes $\nstar \propto \sqrt{B/\log B}$ training states, each measured $\Sstar \propto \sqrt{B \log B}$ times.

The guarantee applies to fixed or independently selected measurement hypotheses, including post-training evaluation with fresh shots; it does not provide a complete generalization guarantee for fully adaptive shot-noisy training.

Future directions include incorporating hardware noise models (decoherence, readout and gate errors) into the allocation framework, extending the result to adaptive shot-noisy training procedures, deriving tighter effective-dimension bounds for PQC-induced hypothesis subclasses, and validating the allocation principle on physical quantum devices.

\bibliographystyle{unsrtnat}
\bibliography{references}

% ── Declarations ─────────────────────────────────────────────
\bigskip
\noindent\textbf{Data availability.}
The synthetic datasets used in this study are generated using the procedures
described in \cref{sec:experiments}.  The generated data and scripts are
publicly available at \url{https://github.com/ocatak/quantum-measurement-budget-allocation}.

\bigskip
\noindent\textbf{Code availability.}
The PennyLane-based experimental code supporting the findings of this study
is publicly available at \url{https://github.com/ocatak/quantum-measurement-budget-allocation}.

\bigskip
\noindent\textbf{Competing interests.}
The author declares no competing interests.

\bigskip
\noindent\textbf{Funding.}
The author received no specific funding for this work.

\bigskip
\noindent\textbf{AI assistance disclosure.}
During manuscript preparation, language-editing and drafting assistance tools were used to improve clarity and structure.  The author reviewed all content and takes full responsibility for the scientific accuracy, completeness, and integrity of the manuscript.  

\newpage
\appendix

% ============================================================
\section{Proof of Proposition~\ref{prop:idealrad}}
\label{app:prop-idealrad}
% ============================================================

\noindent\textbf{Statement.}  The quantum measurement class $\POVM$ cannot
correlate with arbitrary $\pm 1$ labels better than $\sqrt{d/n}$ in
expectation over the random labeling.

\begin{appprop*}[restated]
$\hat{\mathcal{R}}_n^{(\infty)}(\POVM) \leq \sqrt{d/n}$.
\end{appprop*}

\begin{proof}
\textbf{Step 1: Reduction to a single matrix.}
For any fixed sign vector $\boldsymbol{\sigma}$ and any $M$, linearity of
the trace yields:
\[
  \frac{1}{n}\sum_{i=1}^n \sigma_i \Tr(M\rho_i)
  = \Tr\!\left(M \cdot \underbrace{\frac{1}{n}\sum_{i=1}^n \sigma_i\rho_i}_{
    =:\,\bar{\rho}_{\boldsymbol{\sigma}}}\right)
\]
The signed-average matrix $\bar{\rho}_{\boldsymbol{\sigma}}$ is Hermitian
but not in general positive semidefinite, since $\sigma_i \in \{-1,+1\}$.

\medskip\noindent\textbf{Step 2: Inner optimization.}
For a fixed Hermitian matrix $A$, the optimization
$\max_{0\leq M \leq I} \Tr(MA)$ is solved by the projector $\Pi_+(A)$ onto
the positive eigenspace of $A$.
The attained maximum satisfies $\|A\|_{1,+} \leq \|A\|_1$.  Therefore:
\[
  \sup_{M\in\POVM} \Tr(M\bar{\rho}_{\boldsymbol{\sigma}})
  \leq \|\bar{\rho}_{\boldsymbol{\sigma}}\|_1
\]

\medskip\noindent\textbf{Step 3: Nuclear-to-Frobenius norm inequality.}
The Cauchy--Schwarz inequality applied to the vector of singular values gives
$\|A\|_1 \leq \sqrt{d}\|A\|_F$ for any $d\times d$ matrix (tight when all
singular values are equal; potentially loose for low-rank $A$).  Therefore:
\[
  \|\bar{\rho}_{\boldsymbol{\sigma}}\|_1
  \leq \sqrt{d}\|\bar{\rho}_{\boldsymbol{\sigma}}\|_F
\]

\medskip\noindent\textbf{Step 4: Expected Frobenius norm.}
Each $\rho_i$ satisfies $\|\rho_i\|_F^2 = \Tr(\rho_i^2) \leq 1$.
For i.i.d.\ $\pm 1$ signs: $\mathbb{E}[\sigma_i\sigma_j]=0$ ($i\neq j$),
$\mathbb{E}[\sigma_i^2]=1$.  Therefore:
\[
  \mathbb{E}_{\boldsymbol{\sigma}}\!\left[
    \|\bar{\rho}_{\boldsymbol{\sigma}}\|_F^2
  \right]
  = \frac{1}{n^2}\sum_{i=1}^n \|\rho_i\|_F^2
  \leq \frac{1}{n}
\]
Jensen's inequality gives
$\mathbb{E}[\|\bar{\rho}_{\boldsymbol{\sigma}}\|_F] \leq 1/\sqrt{n}$.

\medskip\noindent\textbf{Combining:}
\[
  \hat{\mathcal{R}}_n^{(\infty)}(\POVM)
  \leq \mathbb{E}[\|\bar{\rho}_{\boldsymbol{\sigma}}\|_1]
  \leq \sqrt{d} \cdot \mathbb{E}[\|\bar{\rho}_{\boldsymbol{\sigma}}\|_F]
  \leq \sqrt{d}/\sqrt{n} = \sqrt{d/n}
\]
\end{proof}

% ============================================================
\section{Proof of Theorem~\ref{thm:main}}
\label{app:thm-main}
% ============================================================

\noindent\textbf{Proof strategy.}  Two independent concentration events,
one per error source, combined via a union bound.

\begin{apptheorem*}[restated]
Let $\ell:[0,1]\times\{0,1\}\to[0,1]$ be $L$-Lipschitz in its first argument, and let $M\in\POVM$ be chosen independently of the shot outcomes. With probability $\geq 1-\delta$ over the sample and shots:
$R_\ell(M) \leq \hat{R}_{n,S,\ell}(M) + 2L\sqrt{d/n}
+ L\sqrt{\log(4n/\delta)/(2S)} + \sqrt{\log(2/\delta)/(2n)}$.
\end{apptheorem*}

\begin{proof}
Let $\delta_1, \delta_2 > 0$ with $\delta_1 + \delta_2 = \delta$.

\medskip\noindent\textbf{Event $\mathcal{E}_1$: sample generalization
(probability $\geq 1-\delta_1$).}
The standard Rademacher generalization theorem, which applies because $\ell$ takes values in $[0,1]$ (McDiarmid's inequality requires this), gives, simultaneously for all $M\in\POVM$:
\[
  R_\ell(M) \leq \hat{R}_{n,\ell}(M)
  + 2\hat{\mathcal{R}}_n(\ell \circ \calF_{\POVM})
  + \sqrt{\log(1/\delta_1)/(2n)}
\]
The Rademacher contraction lemma~\cite{ledoux1991probability} gives:
\[
  \hat{\mathcal{R}}_n(\ell \circ \calF_{\POVM})
  \leq L \cdot \hat{\mathcal{R}}_n^{(\infty)}(\calF_{\POVM})
  \leq L\sqrt{d/n}
\]
Since this bound holds uniformly over $\POVM$, it also holds for a hypothesis $M$ selected using the training states and labels. On $\mathcal{E}_1$:
\begin{equation}
  R_\ell(M) \leq \hat{R}_{n,\ell}(M) + 2L\sqrt{d/n}
  + \sqrt{\log(1/\delta_1)/(2n)}
  \label{eq:app-sample}
\end{equation}

\medskip\noindent\textbf{Event $\mathcal{E}_2$: shot concentration
(probability $\geq 1-\delta_2$).}
We condition on the training sample and on $M$. Since the shot outcomes are independent of the selection of $M$, each $\hat{p}_{M,S}(\rho_i)$ is the sample mean of $S$ i.i.d.\
Bernoulli$(p_M(\rho_i))$ variables.  Hoeffding's inequality and a union
bound over $i=1,\ldots,n$:
\[
  \max_i |p_M(\rho_i) - \hat{p}_{M,S}(\rho_i)|
  \leq \sqrt{\log(2n/\delta_2)/(2S)}
\]
The Lipschitz property of $\ell$ transfers this to a loss-level bound, since $|\hat{R}_{n,\ell}(M)-\hat{R}_{n,S,\ell}(M)|\le\frac{1}{n}\sum_i L\,|p_M(\rho_i)-\hat{p}_{M,S}(\rho_i)|$:
\begin{equation}
  \hat{R}_{n,\ell}(M) \leq \hat{R}_{n,S,\ell}(M)
  + L\sqrt{\log(2n/\delta_2)/(2S)}
  \label{eq:app-shot}
\end{equation}

\medskip\noindent\textbf{Combining.}
On $\mathcal{E}_1 \cap \mathcal{E}_2$ (probability $\geq 1-\delta$),
substituting~\eqref{eq:app-shot} into~\eqref{eq:app-sample} and setting
$\delta_1=\delta_2=\delta/2$ yields the stated bound; in particular $\log(2n/\delta_2)=\log(4n/\delta)$, which is the constant used in~\eqref{eq:tradeoff} and in the allocation rule.
\end{proof}

% ============================================================
\section{Proof of Theorem~\ref{thm:allocation}}
\label{app:thm-allocation}
% ============================================================

\noindent\textbf{Setup.}  Throughout this appendix we set $L=1$; for general $L$, both terms are multiplied by $L$, which leaves the minimizer unchanged and scales the minimum value by $L$. The original objective to minimize is
$\gapbound(n) = 2\sqrt{d}/\sqrt{n} + \sqrt{n\log(4n/\delta)/(2B)}$ over $n$
with $S = B/n$.  Because $\log(4n/\delta)$ depends on $n$, this objective
does not admit a closed-form minimizer.  We instead minimize the
surrogate $\gapbound_C(n) = 2\sqrt{d}/\sqrt{n} + \sqrt{n\log(4B/\delta)/(2B)}$
obtained by replacing $\log(4n/\delta)$ by $\log(4B/\delta)$.  For $n \leq B$,
this substitution produces an upper bound on the original objective
(since $\log(4n/\delta) \leq \log(4B/\delta)$ when $n \leq B$), so the
surrogate minimizer provides a valid allocation for the original
problem, with the same asymptotic scaling in $B$.

\begin{apptheorem*}[restated]
The minimizer of the surrogate objective $\gapbound_C(n)$ is
$\nstar = 2\sqrt{2dB/\log(4B/\delta)}$, and at this allocation
$\gapbound_C(\nstar) = O\!\left((d\log(4B/\delta)/B)^{1/4}\right)$.  The exact
minimizer of the original objective $\gapbound(n)$ satisfies the implicit
first-order condition
\[
  n\bigl(\log(4n/\delta) + 1\bigr) = 2\sqrt{2dB\log(4n/\delta)},
\]
and can be determined numerically.
\end{apptheorem*}

\begin{proof}
\textbf{Step 1: U-shape of the surrogate.}
Let $C_\delta := \log(4B/\delta)$.  Then $f_1(n) = 2\sqrt{d/n}$ is strictly
decreasing and $f_2(n) = \sqrt{nC_\delta/(2B)}$ is strictly increasing, so
$\gapbound_C$ has a unique interior minimum.

\medskip\noindent\textbf{Step 2: First-order condition for the surrogate.}
Write $\gapbound_C(n) = an^{-1/2} + bn^{1/2}$ with $a = 2\sqrt{d}$ and
$b = \sqrt{C_\delta/(2B)}$.  Setting $\gapbound_C'(n) = 0$:
\[
  -\tfrac{a}{2}n^{-3/2} + \tfrac{b}{2}n^{-1/2} = 0
  \;\Longrightarrow\; n = a/b = 2\sqrt{2dB/C_\delta}
\]

\medskip\noindent\textbf{Step 3: Evaluation at the surrogate optimum.}
At $n = \nstar$:
\[
  f_1(\nstar) = f_2(\nstar) = 2^{1/4}\!\left(\frac{dC_\delta}{B}\right)^{\!1/4},
  \qquad
  \mathcal{G}_C(\nstar) = 2^{5/4}\!\left(\frac{dC_\delta}{B}\right)^{\!1/4}
\]

\medskip\noindent\textbf{Step 4: Exact first-order condition for the original objective.}
Differentiating $\gapbound(n) = 2\sqrt{d}/\sqrt{n} +
\sqrt{n\log(4n/\delta)/(2B)}$ with respect to $n$ and setting the derivative
to zero gives:
\[
  -\frac{\sqrt{d}}{n^{3/2}} + \frac{1}{2\sqrt{2B}}\cdot
  \frac{\log(4n/\delta) + 1}{\sqrt{n\log(4n/\delta)}} = 0,
\]
which simplifies to
\[
  n\bigl(\log(4n/\delta) + 1\bigr) = 2\sqrt{2dB\log(4n/\delta)}
\]
This implicit equation can be solved numerically (e.g., via bisection on
$[1, B]$).

\medskip\noindent\textbf{Step 5: Validity of the surrogate as a planning rule.}
For $\nstar \leq B$, which requires $8d/C_\delta \leq B$ (satisfied for all
practical parameter ranges), $\gapbound_C(\nstar)$ upper-bounds
$\gapbound(\nstar)$, confirming that the surrogate allocation is conservative
with respect to the original bound.
\end{proof}

\end{document}